\documentclass{article}
\usepackage{ijcai22}

\usepackage{definitions}

\usepackage{times}
\usepackage{soul}
\usepackage{url}
\usepackage[hidelinks]{hyperref}
\usepackage[utf8]{inputenc}
\usepackage[small]{caption}
\usepackage{xfrac}
\usepackage{graphicx}
\usepackage{amsmath}
\usepackage{amsthm}
\usepackage[ruled,vlined,shortend, linesnumbered]{algorithm2e} 
\usepackage{algpseudocode}
\usepackage[table,xcdraw]{xcolor}

\usepackage{booktabs}
\usepackage{placeins}
\newsavebox{\measurebox}
\usepackage{pgfplots}
\usepackage{tikz}
\usetikzlibrary{automata,positioning}
\usepackage{adjustbox}

\usepackage{subcaption}
\usepackage{relsize}
\usepackage{bbold}
\newcommand{\deepmethod}{\textsc{ExpertNet}}

\makeatletter
\newcommand*{\inlineequation}[2][]{%
  \begingroup
    \refstepcounter{equation}%
    \ifx\\#1\\%
    \else
      \label{#1}%
    \fi
    \relpenalty=10000 %
    \binoppenalty=10000 %
    \ensuremath{%
      #2%
    }%
    ~\@eqnnum
  \endgroup
}
\makeatother

\title{\deepmethod: A Symbiosis of Classification and Clustering}

\author{
Shivin Srivastava$^1$
\and
Kenji Kawaguchi$^1$ \And
Vaibhav Rajan$^1$
\affiliations
$^1$ National University of Singapore\\
\emails
\{shivin, kenji, vaibhav\}@comp.nus.edu.sg
}

\begin{document}

\maketitle

\begin{abstract}

A widely used paradigm to improve the generalization performance of high-capacity neural models is through the addition of auxiliary unsupervised tasks during supervised training.
Tasks such as similarity matching and input reconstruction have been shown to provide a beneficial regularizing effect by guiding representation learning.
Real data often has complex underlying structures and may be composed of heterogeneous subpopulations that are not learned well with current approaches.
In this work, we design \textbf{\deepmethod}, which uses novel training strategies to learn clustered latent representations and leverage them by effectively combining cluster-specific classifiers.
We theoretically analyze the effect of clustering on its generalization gap, and empirically show that clustered latent representations from \deepmethod\ lead to disentangling the intrinsic structure and improvement in classification performance.
\deepmethod\ also meets an important real-world need where classifiers need to be tailored for distinct subpopulations, such as in clinical risk models.
We demonstrate the superiority of \deepmethod\ over state-of-the-art methods on $6$ large clinical datasets, where our approach leads to valuable insights on group-specific risks.

\end{abstract}

\section{Introduction}
\label{sec:introduction}

Overfitting is a common problem in high-capacity 
models, such as 
neural networks,
that adversely affects generalization performance.
A standard approach to reduce overfitting is regularization.
Various regularization strategies have been developed
such as those based on norm penalties and some specifically designed for neural networks such as dropout and early stopping \cite{goodfellow2016deeplearning}.
The use of auxiliary tasks have also been explored for regularization.
The rationale comes from multi-task learning, which is known to improve generalization through improved statistical strength via shared parameters \cite{baxter1995learning,caruana1997multitask}.

Since there is no dependence on label acquisition, several unsupervised auxiliary tasks have been explored for regularizing neural networks.
Input reconstruction is a common choice, often implemented using autoencoder-like architectures, e.g. \cite{rasmus2015ladder,zhao2015swwae,zhang2016augmenting,leli2018supervised}. 
Such reconstruction-based regularization has the advantage of yielding hidden layer representations closer to the original data topology.



Real data, however, contains complex underlying structures like clusters and intrinsic manifolds that are not disentangled by the bottleneck layer of autoencoders. 
This has been studied in the context of purely unsupervised models that have explored the 
joint tasks of clustering and dimensionality reduction (DR) to find ``cluster-friendly'' representations.
E.g., in 
\cite{yang2017towards}, it is shown that joint clustering and DR outperforms the disassociated approach of independently performing DR followed by clustering.  
Recognizing the importance of preserving cluster structure in latent embedded spaces, we hypothesize that, in addition to reconstruction,
clustering-based regularization could improve generalization performance in supervised models.

Our work is also motivated by the need to develop models tailored to distinct subpopulations in the data.
This is a common requirement in clinical risk models since 
patient populations can show significant 
heterogeneity. 
Most previous works adopt a `cluster-then-predict' approach to model such data
e.g., \cite{ibrahim2020classifying}, where clustering is first done independently to find
subpopulations (called subtypes) 
and then predictive models for each of these clusters are learnt.
Such subtype-specific models are often found to outperform models learnt from the entire population
\cite{masoudnia2014mixture}.
They also provide a form of interpretability and subsequent clinical decisions can be personalized to each patient group's characteristics.
However, in these approaches, clustering is performed independent of classifier training, and may not discover latent structures that are beneficial for subsequent classification.
This leads us to hypothesize that combined clustering 
with supervised classification will 
lead to better performance compared to such cluster-then-predict approaches.

We thus design \deepmethod, a deep learning model that performs simultaneous clustering and classification. \deepmethod\ consists of components that have both local and global view of the data space. The local units specialize in classifying observations in clusters found in
the data while the global unit is responsible for generating cluster friendly embeddings based on the feedback given by local units. Apart from being a supervised model, \deepmethod\ can function as an unsupervised model that can perform target-specific clustering.
In summary, our contributions are:
\begin{enumerate}[noitemsep,topsep=0pt,labelindent=0em,leftmargin=*]
\item \textbf{Model}: We design \deepmethod, a neural model for simultaneous clustering and classification with cluster-specific local networks. We introduce novel training strategies to obtain latent clusters and use them effectively to improve generalization of the classifiers.
\item \textbf{Theoretical Guarantees}: We analyze the effect of clustering on \deepmethod's generalization gap.
\item \textbf{Experiments}: Our extensive experiments  demonstrate the efficacy of our model over both reconstruction-based regularization and cluster-then-predict approaches. 
On $6$ large real-world clinical datasets
\deepmethod\  
achieves clustering performance that is comparable to, and achieves classification performance that is considerably better than, state-of-the-art methods.
\end{enumerate}

\section{Background and Related Work}
\label{sec:related_work}

Unsupervised learning techniques are often used to pretrain a neural network before finetuning with supervised tasks \cite{bengio2007greedy}.
Input reconstruction 
is commonly 
used to regularize deep models. 
\cite{zhang2016augmenting} jointly train a supervised neural model with an unsupervised reconstruction network.
\cite{zhao2015swwae} showed that by adding decoding pathways, to add an auxiliary reconstruction task to the objective, improves the performance of supervised networks.
Similarly \cite{leli2018supervised} append a prediction layer to the embedding layer to improve generalization performance.

A more dissociated way of using clustering to benefit supervised models is to simply find clusters in the data and then train separate models on each of them that. Amongst such techniques, there are some algorithms \cite{finley2008supervised,gu2013clustered,fu2010mixing} that impose a common, global regularization scheme on all the individual classifiers. In some other works \cite{reyna2019sepsis,lasko2013computational,suresh2018learning}, all the individual classifiers are completely independent.
Such analysis is common in clinical data where patient stratification is done followed by predictive modeling.
E.g., Deep Mixture Neural Networks (DMNN) \cite{li2020dmnn}, designed with this aim, is based on an autoencoder architecture while having an additional gating network, which is first trained to learn clusters through $k$-means.
Local networks corresponding to each cluster are trained using a weighted combination of losses, where the weights are determined by the gates.

Simultaneous clustering and representation learning have been used in purely unsupervised settings.
Examples include neural clustering models such as
Deep Embedded Clustering (DEC)
\cite{xie2016unsupervised} and Improved DEC (IDEC) \cite{Guo2017IDEC} and Deep Clustering Network (DCN) \cite{yang2017towards}.
Recent self-supervised learning techniques also employ cluster centroids as pseudo class labels in embedded data space to improve the quality of representations \cite{li2020prototypical,caron2018deep}. 

We briefly describe DEC
\cite{xie2016unsupervised} as the loss function in \deepmethod\ is based on their formulation. 
Their key idea is to find latent representations $z_i$ for data points using an autoencoder by first pretraining it.
Initial cluster centers $\{\mu_j\}$ are obtained by 
using $k$-means on the latent representations.
The decoder is then removed and the representations are fine tuned by minimizing the loss function:
\begin{align}
    \label{eqn:kl}
    L_c = KL(P || Q) = \sum_{i=1}^{N} \sum_{j=1}^{N} p_{ij} \log\frac{p_{ij}}{q_{ij}}
\end{align}
where $q_{ij}$ is the probability of assigning the $i^{\rm th}$ data point ($z_i$) to the $j^{\rm th}$ cluster (with centroid $\mu_j$) measured by the Student's $t$-distribution \cite{maaten2008visualizing} as:
\begin{align}
    \label{eqn:q}
    q_{ij} = \frac{(1+\|z_i-\mu_j\|^2)^{-1}}{\sum_j(1+\|z_i-\mu_j\|^2)^{-1}}
\end{align}
The target distribution $p_{ij}$ is defined as
\begin{align}
    \label{eqn:p}
    p_{ij} = \frac{\sfrac{q_{ij}^2}{\sum_i q_{ij}}}{\sum_j (\sfrac{q_{ij}^2}{\sum_i q_{ij}})}
\end{align}
The predicted label of $x_i$ is $\argmax_j q_{ij}$. 
Minimizing the loss function $L_c$ is a form of self-training as points with high confidence act as anchors and distribute other points around them more densely.

DEC, IDEC and DCN are all centroid-based and have objectives similar to that of $k$-means.
In such approaches, the encoder can map the centroids to a single point to make the loss zero and thus collapse the clusters.
Hence, previous methods employ various heuristics to balance cluster sizes, e.g., through sampling approaches \cite{caron2018deep} or use of priors
\cite{Jitta2018OnCT}.

\section{\deepmethod}
\label{sec:ExpertNet_method}
Given $N$ datapoints $\{x_i \in X\}^N_{i=1}$, and 
class labels $y_i \in \{1, \dots, \mathcal{B}\}$, associated with every point $x_i$,
our aim is to simultaneously
(i) cluster the $N$ datapoints  into $k$ clusters, each represented by a centroid $\mu_j , j \in \{1,\dots , k\}$, and
(ii) build $k$ distinct supervised classification models within each cluster to predict the class labels.
Note that during training, labels are used only for building the classification models and not for clustering.

\subsubsection{Network Architecture}

Fig. \ref{fig:ExpertNet} shows the neural architecture of \deepmethod\ that consists of an encoder, a decoder and $k$ local networks ($LN_j$).
The encoder $f(\Ucal): X \xrightarrow{} Z$ is used to obtain low-dimensional representations of the input datapoints.
Cluster structure is learnt in this latent space and representations in each cluster are used in local networks, $h(\Wcal_j): Z \xrightarrow{} \hat{y}_j$ for $j=1 \ldots k$ to train $k$ classification models.
In addition, there is a decoder 
$g(\Vcal): Z \xrightarrow{} X$ that is used to reconstruct the input from the embeddings.
\deepmethod\ is parameterized by three sets of weights: $\Ucal, \Vcal, \{\Wcal_j\}_{j=1}^{k}$ which are learnt by optimizing a combination of losses as described below.

\begin{figure}[h]
    \includegraphics[width=\columnwidth]{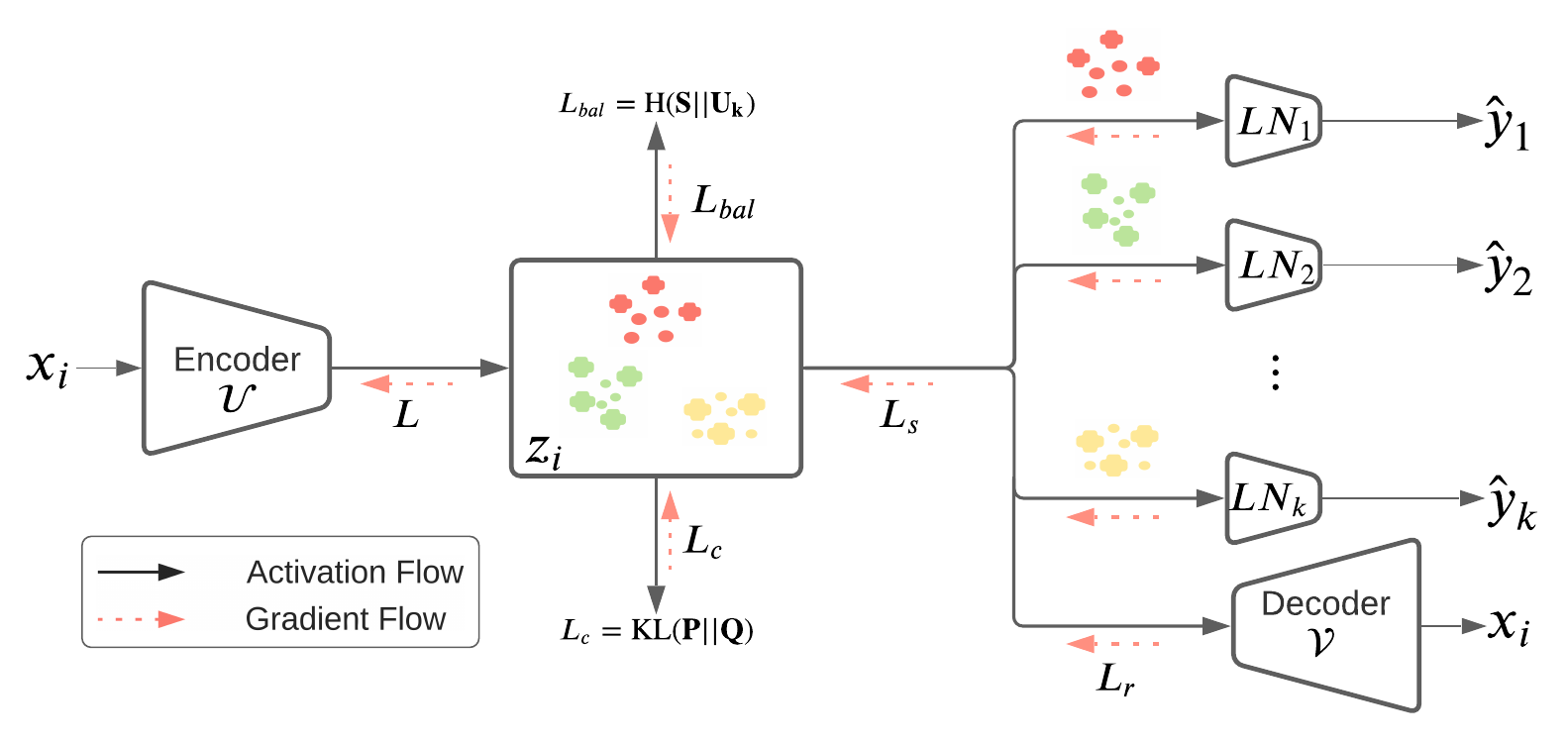}
    \caption{\deepmethod\ Architecture}
    \label{fig:ExpertNet}
\end{figure}

\subsubsection{Loss Function}

The overall loss function $L$ is 
a weighted combination, with coefficients $\beta, \gamma>0$, of the reconstruction loss $L_r$, clustering loss 
$L_c$, cluster balance loss $L_{bal}$ and classification loss $L_s$:
\begin{align}
L = L_r + \beta \cdot L_c + \gamma \cdot L_s + \delta \cdot L_{bal} \label{eqn:ExpertNet_loss}
\end{align}

$L_c$ is defined as the KL divergence loss (Eq. \ref{eqn:kl}), where the cluster membership distribution $Q$ (Eq. \ref{eqn:q}) uses representations $z_i$ and cluster centroids $\mu_j$ inferred during \deepmethod\ training.
As suggested in \cite{Guo2017IDEC,yang2017towards},
to prevent distortion in the latent space and improve clustering performance, 
we add the reconstruction loss measured by mean squared error:
\begin{equation}
\label{eqn:lr}
L_r = \sum_{i=1}^{N} \|x_i - g(f(x_i))\|^{2}
\end{equation}

We design a novel cluster balance loss to discourage unevenly distributed cluster sizes. We define the `soft' size of a cluster $C_j$ as $|C_j| = \sum_{i} q_{ij}$ and cluster support counts $\mathbf{S} = [|C_1|, |C_2|, \dots, |C_k|]$ as a $k$-dimensional probability distribution.
Let $\mathbf{U_k} = (\sfrac{1}{k})\mathbf{U}(0,1)$ denote the $k$-dimensional uniform distribution function. 
We use the Hellinger distance ($H$), which measures the dissimilarity between two distributions, as the loss
${L_{bal} = H(\mathbf{S} || \mathbf{U_k}) = \frac{1}{2}\|\sqrt{S} - \sqrt{U_k}\|_2}.$
$L_s$ is a weighted cross entropy loss described in the following.

\subsubsection{\deepmethod\ Training}

To initialize the parameters $(\Ucal, \Vcal)$, we first pre-train the encoder and decoder with the input data using only the reconstruction loss $L_r$. This is followed by $k$-means clustering on $\{z_i = f(x_i; \Ucal)\}^{N}_{i=1}$ to obtain cluster centroids $\{\mu_i\}_{i=1}^{k}$ which are used to calculate  cluster membership and target distributions, $Q,P$.
After initialization, 
we use mini-batch Stochastic Gradient Descent to train the entire network, using the loss $L$ (details are in Appendix \ref{app:optimization}). 
To stabilize training we update $P$ only after every epoch.

\begin{algorithm}[!hbt]
    \caption{ExpertNet\ Training\label{algo:cac}}
    \KwIn{Training Data: $X \in \mathbb{R}^{n\times d}$, labels $y^{n \times 1} \in [\mathcal{B}]^{n}$, $k$, $f(\cdot;\Ucal), g(\cdot;\Vcal)$ and $\{h_j(\cdot;\Wcal_j)\}_{j=1}^{k}$, $\Jcal$}
    {\bf $\triangleright$ Initialization} \\
    Pre-train $f(\cdot; \Ucal)$ \& $g(\cdot; \Vcal)$ via back-propagating loss in eq. \ref{eqn:lr}\\
    Compute $\mu(C_j) \forall\ C_j\ s.t.\ j \in [k]$ and $b \in [\mathcal{B}]$\\
    Compute matrices $Q$ and $P$ according to eqs. \ref{eqn:q} and \ref{eqn:p}

    {\bf $\triangleright$ Algorithm} \\
    \While{Validation AUC increases}{
        \For{every mini-batch $\Xcal_b$}{
            $\Zcal_b \xleftarrow{} f(\Xcal_b;\Ucal)$\\
            Calculate $Q_b$ by eq. (\ref{eqn:q})\\
            \For{$T$ sub-iterations}{
                Sample $\{s_i\sim q_{i\cdot}\}_{i=1}^{|\Xcal_b|}$\\
                Calculate $\{\Xcal_m = \{x_i:s_i=j\}_{i=1}^{|\Xcal_b|}\}_{j=1}^{k}$\\
                Train local classifiers $\{h_j\}_{j=1}^{k}$ on $(\Xcal_j, \Ycal_j)$\\
                Update $\{\Wcal_j\}_{j=1}^{k}$
            }
            Backpropagate $L = L_r + \beta\cdot L_c + \gamma \cdot L_s + \delta \cdot L_{bal}$ and update $\mu, \Ucal, \Vcal$ and $\Wcal$ 
        }
        Update $P$ via eq. (\ref{eqn:p})
    }
    {\bf $\triangleright$ Fine tune Local Classifiers} \\
    For every cluster $C_j$, train a classifier $f_j$ on $(\Xcal_{j}, \Ycal_{j})$.

    {\bf Output} Trained \deepmethod\ model. Cluster centroids $\mu = \{\mu_j\}_{j=1}^{k}$
\end{algorithm}

Since the embeddings get updated progressively in every iteration, we train the LNs for a larger number of {\it sub-iterations} within every iteration of the main training loop. 
The number of sub-iterations gradually increases ($1$ per every 5 epochs until a max-limit we set to $10$).
This enables the LNs to learn better from the stabilized clustered embeddings than from the 
intermediate representations.

Further, we design a novel strategy called {\it stochastic cohort sampling} used in each sub-iteration,
which leads to more robust classifiers, as verified in our experiments.
Instead of training each LN with a fixed cluster of data points (e.g., using 
$\argmax_{j=1}^{k} q_{i,j}$),
we leverage the  probabilistic definition of clusters to obtain multiple, different cluster assignments for the same set of embeddings. Considering $q_{i\cdot}$, the $i^{th}$ row of $Q$, as a cluster probability distribution for $z_i$, we sample a random variable $s_i$ from $q_{i\cdot}$, denoting the cluster assignment for the point $x_i$. We can define a cluster realization $C_m = \{i:s_i=m\}_{i=1}^{N}$ (i.e. index of all the points assigned to cluster $C_m$) and $\Xcal_m = \bigcup\limits_{j\in C_m} \{x_j\}$.
The LNs are trained on these cluster realizations for $T-1$ sub-iterations \textit{without} backpropagating the error to the encoder. The individual errors are collected from all the $k$ LNs only at the last ($T^{\rm th}$) iteration and backpropagated to the encoder to adjust the cluster representations accordingly. The final 
classification loss 
is a weighted cross entropy (CE) loss:
\begin{equation}
L_s = \sum_{j=1}^{k} \sum_{p\in C_j} q_{p,j} \operatorname{CE}(y_p, h_j(x_p; \Vcal_j)).
\end{equation}

After training, the encoder network is frozen and the local networks are finetuned on the latent embeddings via stochastic cohort sampling
to further improve LN performance.

\subsubsection{\deepmethod\ Prediction}

Prediction can be done using the encoder and local networks. 
For a test point $\hat{x}_p$, the soft cluster probabilities ($\hat{q}_{pj}$) are calculated from Eq. \ref{eqn:q}. All the local networks (see stochastic cohort sampling above) can be used to predict the class label 
$
\hat{y}_p = \sum_{j=1}^{k} \hat{q}_{pj} h_j(\hat{x}_p)
$.

\section{Theoretical Analysis}
Let  $\{\Omega_j\}_{j=1}^k$ be a partition of $\Xcal$ such that $\Xcal = \cup_{j=1}^k\Omega_j$ where $\Omega_j \cap \Omega_{j'} = \emptyset$ for $j\neq j'$. This corresponds to the inverse image of clustering at the space of $z$. We define $N_j$ as 
$$
N_j = \sum_{i=1}^N\one\{x _{i}\in \Omega_j\}; f(x) = \sum_{j=1}^k \one\{x \in \Omega_j\} f_j(x)
$$
where $f_j =g_{j} \circ f$ with $g_j$ being the encoder for the cluster $\Omega_j$ and $f$ being the shared encoder.   


While we provide the results for multi-class classification in Appendix \ref{app:1}, this section considers  binary classification. For binary  classification problems with  $y \in \{-1,+1\}$ and $f(x)\in \RR$, define the margin loss as follows: 
 $$
 {\ell}_{\rho}(f(x),y)={\ell}_{\rho}^{(1)}(f(x)y)
 $$  
 where 
$$
{\ell}_{\rho}^{(1)}(q)=
\begin{cases} 0 & \text{if } \rho \le q \\
1-q/\rho &  \text{if } 0 \le q \le \rho \\
1 & \text{if } q \le 0.\\
\end{cases}
$$

Define the 0-1 loss  as:
$$
\textstyle \ell_{01}(f(x),y) = \one\{f(x)y \le 0\}. 
$$

To simplify the equation and discussion, we consider the case where $\Pr(x \in \Omega_j)=1/k$ (uniform) and $N_{j}=N/k$ (uniform), whereas the results for a more general case are presented in the appendix. Defining $\hat \EE_{x,y}[\ell_{\rho}(f(x),y)] $ to denote the empirical los, we have: 
\begin{theorem}
\label{thm:5}
Suppose that for all $j \in \{0,\dots,k\}$, the function  $\sigma_{l}^j$ is 1-Lipschitz and positive homogeneous for all $l \in [\max(L-1, Q-1)]$ and $\|x^{j}\|\le B_j$ for all $x^{j} \in \Omega_j$. Let $\Fcal_j =\{x \in \Omega_j \mapsto (g_{j} \circ e)(x): (\forall l \in[L-1])[\|W_{l}^{j}\|_F \le M_l^{j} \wedge\|W_{l}^{0}\|_F \le M_l ^{0}] \}$ and $\Fcal^{k}= \{x\mapsto f(x):f(x) = \sum_{j=1}^k \one\{x \in \Omega_j\} f_j(x), f_j \in \Fcal_j\}$. Then, for any $\delta>0$, with probability at least $1-\delta$ over  an i.i.d. draw of $m$ i.i.d. test samples  $((x_i, y_i))_{i=1}^m$, the following holds: for all maps $ f^{k}\in\Fcal^{k}$,
\begin{align*}
\EE_{x,y}[\ell_{01}(f^{k}(x),y)] &- \hat \EE_{x,y}[\ell_{\rho}(f^{k}(x),y)]
\le \\
& \frac{\zeta_1\left(\sum _{j=1}^K B_j \prod_{l=1}^L M_l^{j} \right)}{\sqrt{kN}}\\
& + 3\sqrt{\frac{k\ln(2k/\delta)}{2N}}     
\end{align*}
where $\zeta_1=2 \rho^{-1}(\sqrt{2 \log(2) (L+Q) }+1)(\prod_{l=1}^Q M_l^{0})$ is a $k$ independent term. 
\end{theorem}
The proof is presented in Appendix \ref{app:1}. Theorem \ref{thm:5} shows the generalization bound with the following:
\begin{itemize}[noitemsep,topsep=0pt,labelindent=0em,leftmargin=*]
\item 
Increasing $k$ can reduce the empirical loss  $\EE_{x,y}[\ell_{\rho}(f^{k}(x),y)]$ (since it increases the expressive power), resulting in a tendency for a better expected error $\EE_{x,y}[\ell_{01}(f^{k}(x),y)]$.
\item 
Increasing  $k$  increases the last term   $\sqrt{\frac{k\ln(2k/\delta)}{2N}}$ , resulting in  a tendency for a worse expected error $\EE_{x,y}[\ell_{01}(f^{k}(x),y)]$.
\item
Increasing  $k$  can reduce the complexity term  $\frac{\left(\sum _{j=1}^k B_j \prod_{l=1}^L M_l^{j} \right)}{\sqrt{kN}}$ (since (1) each sub-network only needs to learn a simpler classifier with larger $k$, resulting in a smaller value of $\prod_{l=1}^L M_l^{j}$; (2) each domain decrease and hence $B_j$ decrease as $k$ increase), resulting in a tendency for a better expected error $\EE_{x,y}[\ell_{01}(f^{k}(x),y)]$.
\end{itemize}

Not surprisingly, the generalization gap is inversely dependent on $N$, the total number of data points, indicating that more observations
will lead to a more accurate model. There is no simple relationship between $k$ 
suggesting
that the best number of clusters might have to be found empirically.

\begin{figure*}[!htb]
     \centering
     \begin{subfigure}[b]{0.16\textwidth}
         \centering
         \includegraphics[width=\textwidth]{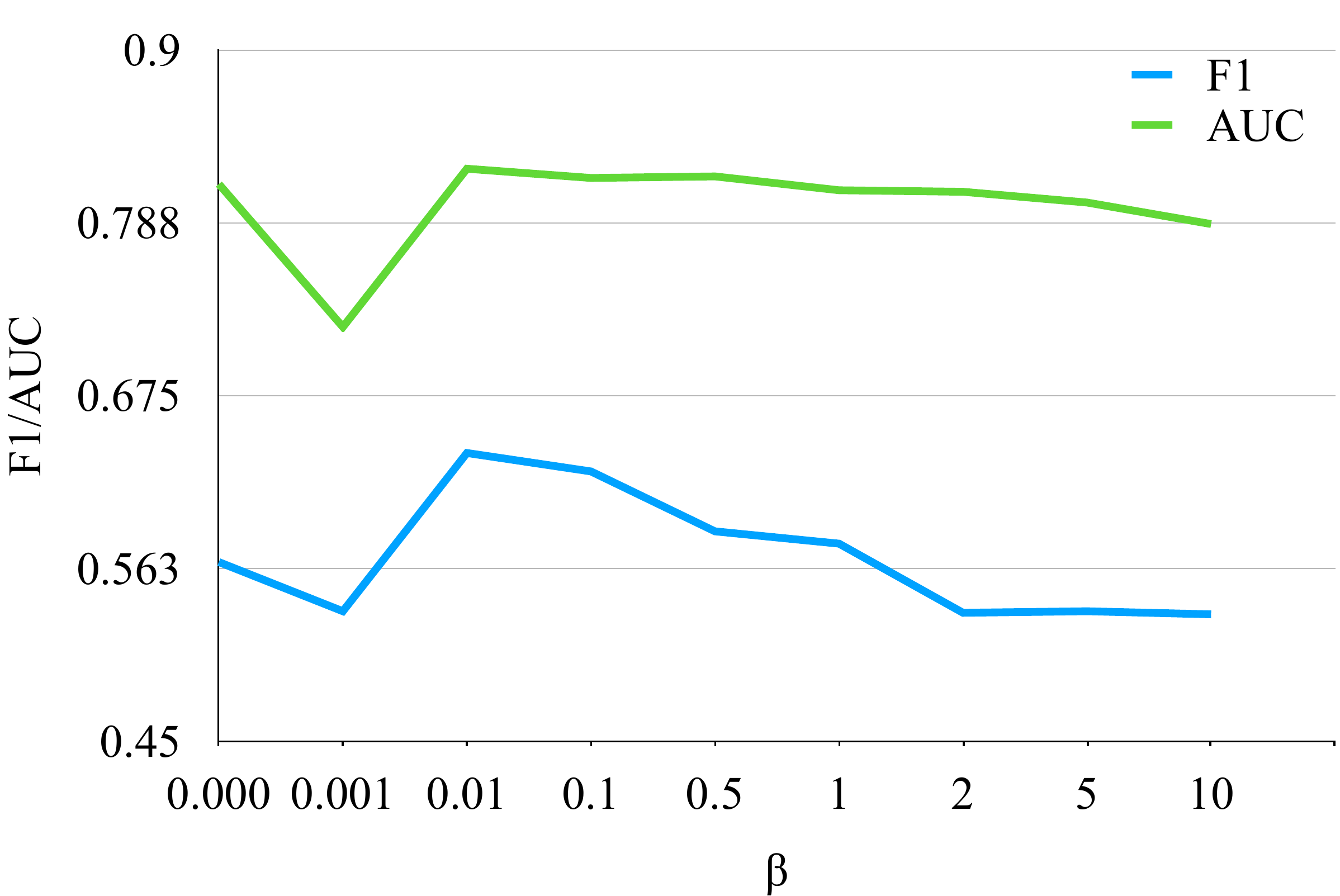}
        \caption{F1/AUC vs. $\beta$}
        \label{fig:f1_auc_beta}
    \end{subfigure}
     \hfill
     \begin{subfigure}[b]{0.16\textwidth}
         \centering
         \includegraphics[width=\textwidth]{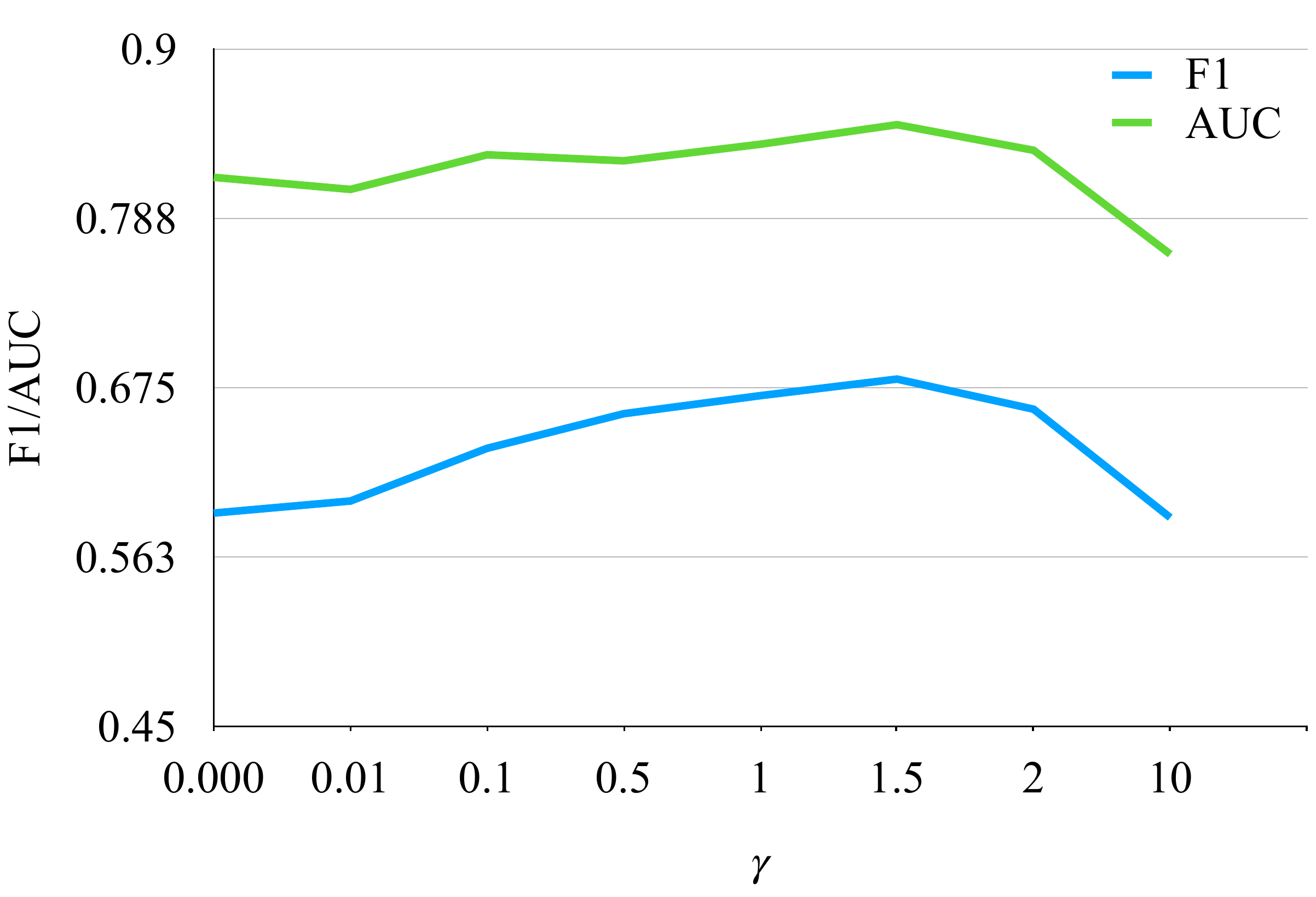}
        \caption{F1/AUC vs. $\gamma$}
        \label{fig:f1_auc_gamma}
     \end{subfigure}
     \hfill
     \begin{subfigure}[b]{0.16\textwidth}
         \centering
         \includegraphics[width=\textwidth]{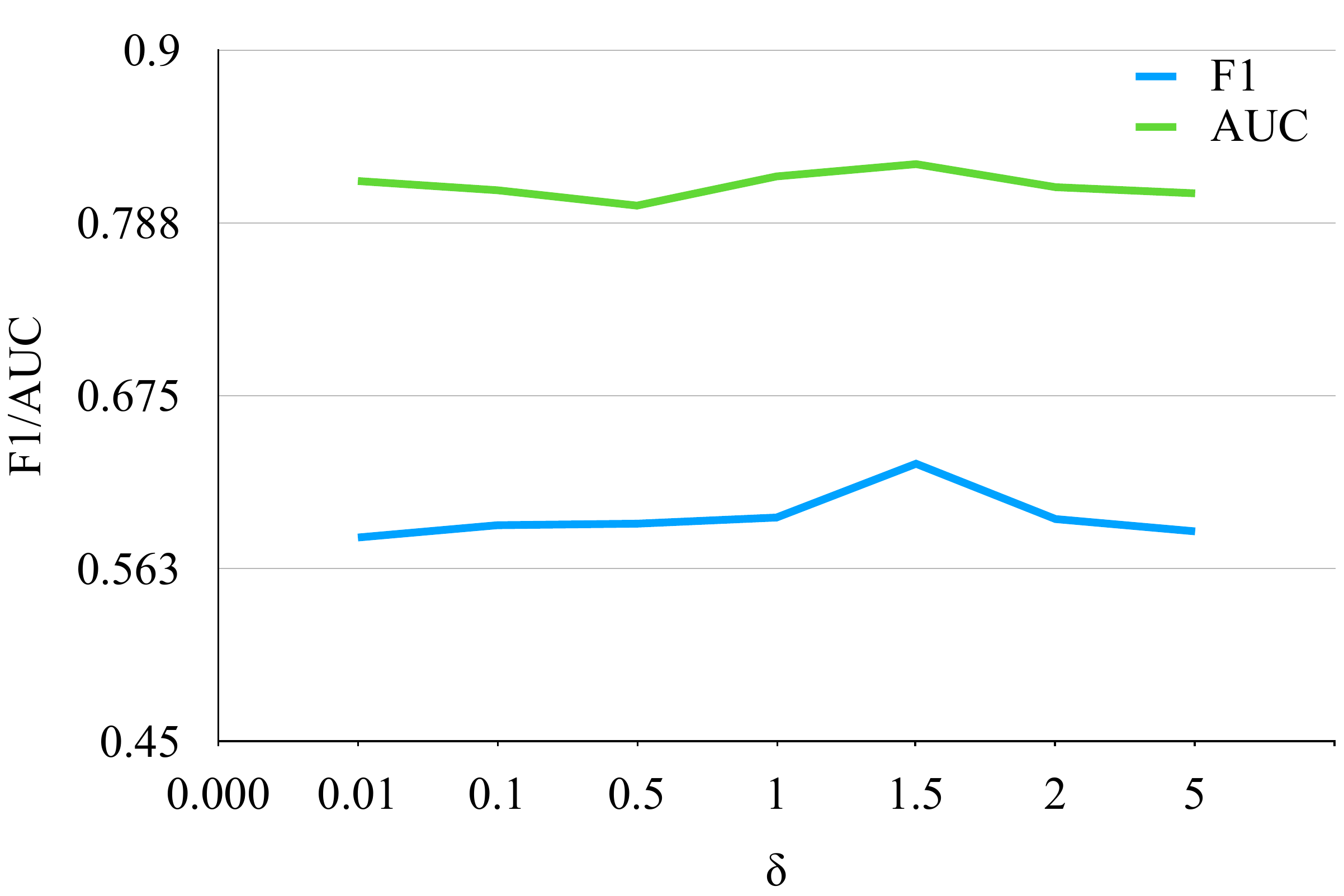}
        \caption{F1/AUC vs. $\delta$}
        \label{fig:f1_auc_delta}
     \end{subfigure}
     \hfill
     \begin{subfigure}[b]{0.16\textwidth}
         \centering
         \includegraphics[width=\textwidth]{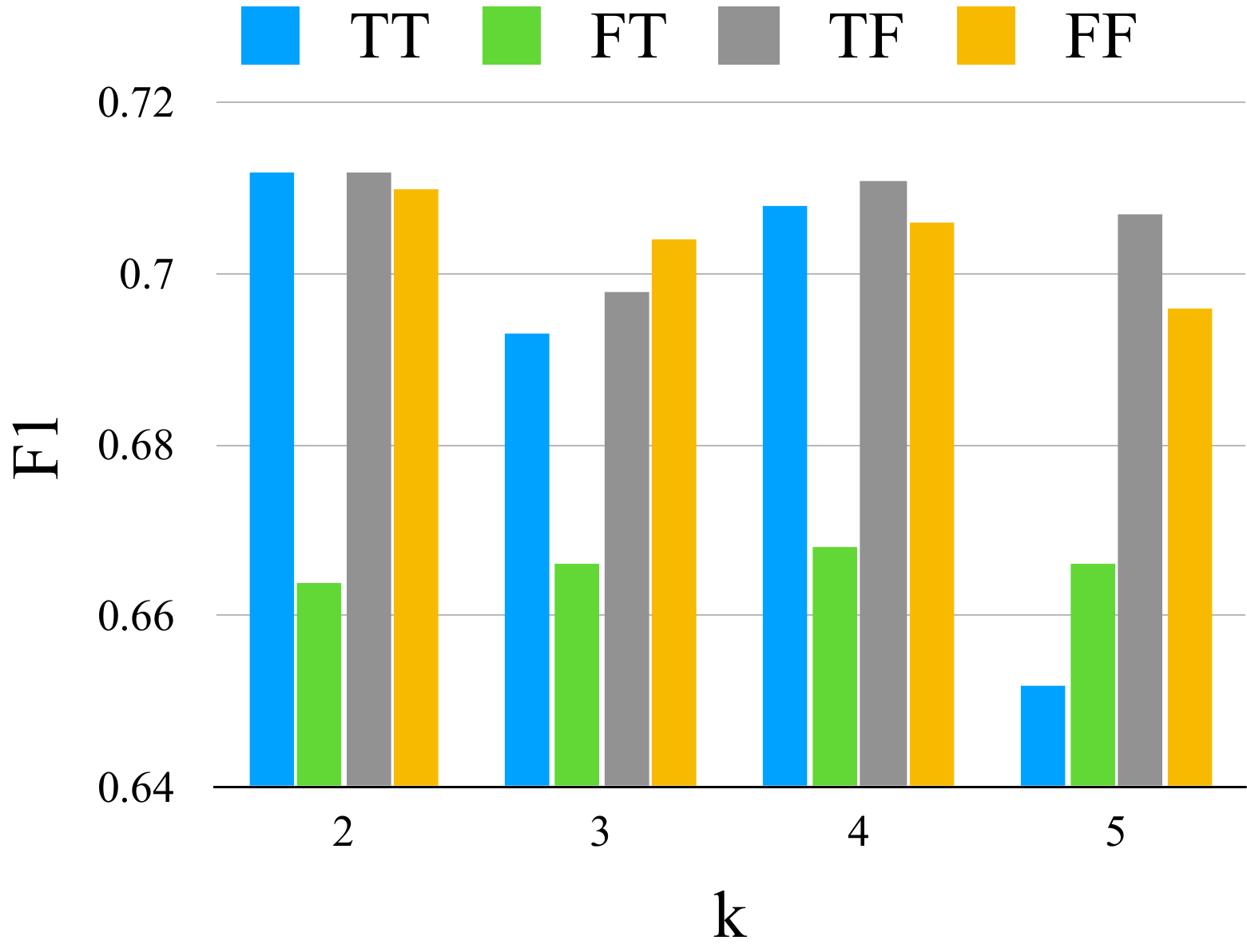}
         \caption{F1}
         \label{fig:f1_cic}
     \end{subfigure}
     \hfill
     \begin{subfigure}[b]{0.16\textwidth}
         \centering
         \includegraphics[width=\textwidth]{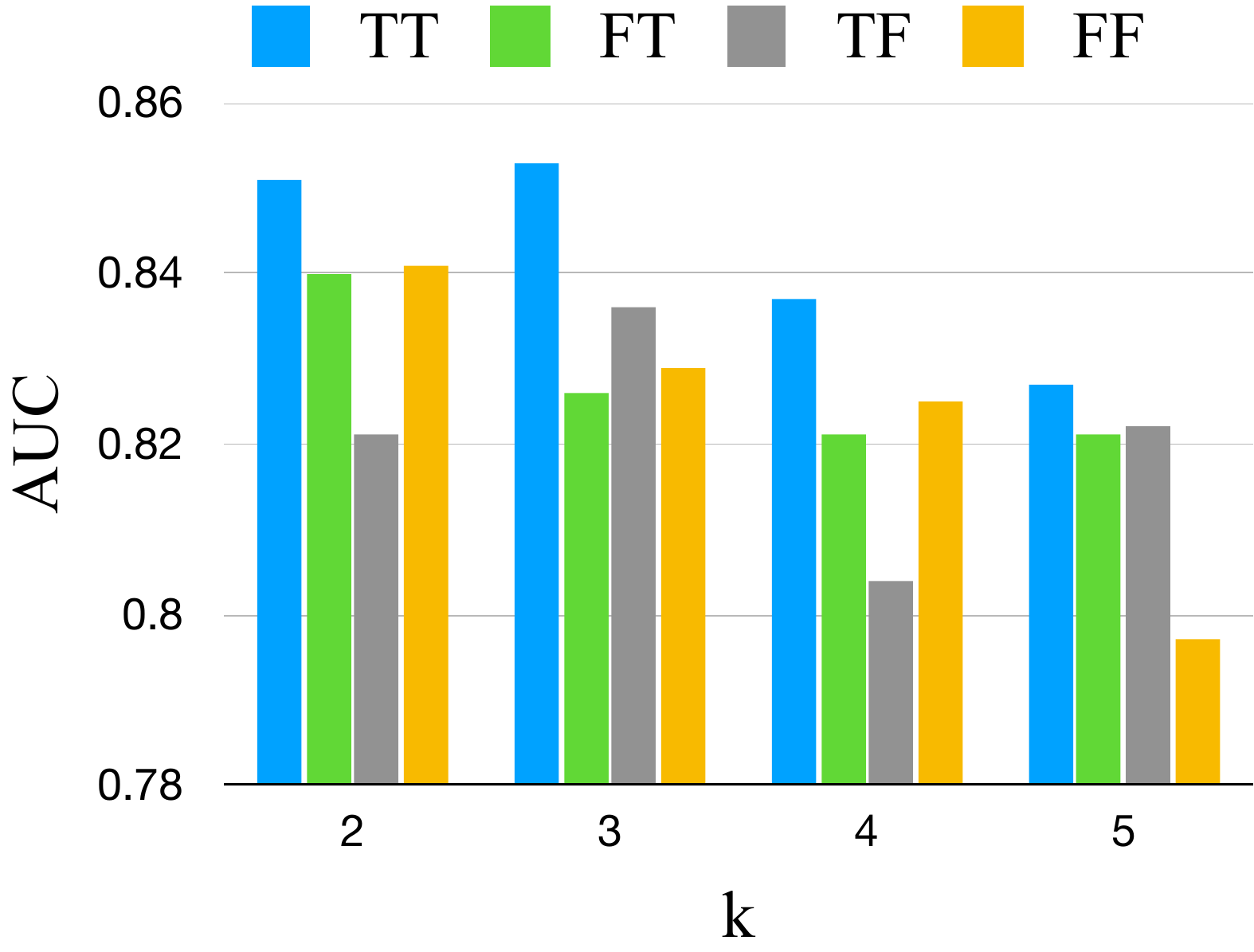}
         \caption{AUC}
         \label{fig:auc_cic}
     \end{subfigure}
     \hfill
     \begin{subfigure}[b]{0.16\textwidth}
         \centering
         \includegraphics[width=\textwidth]{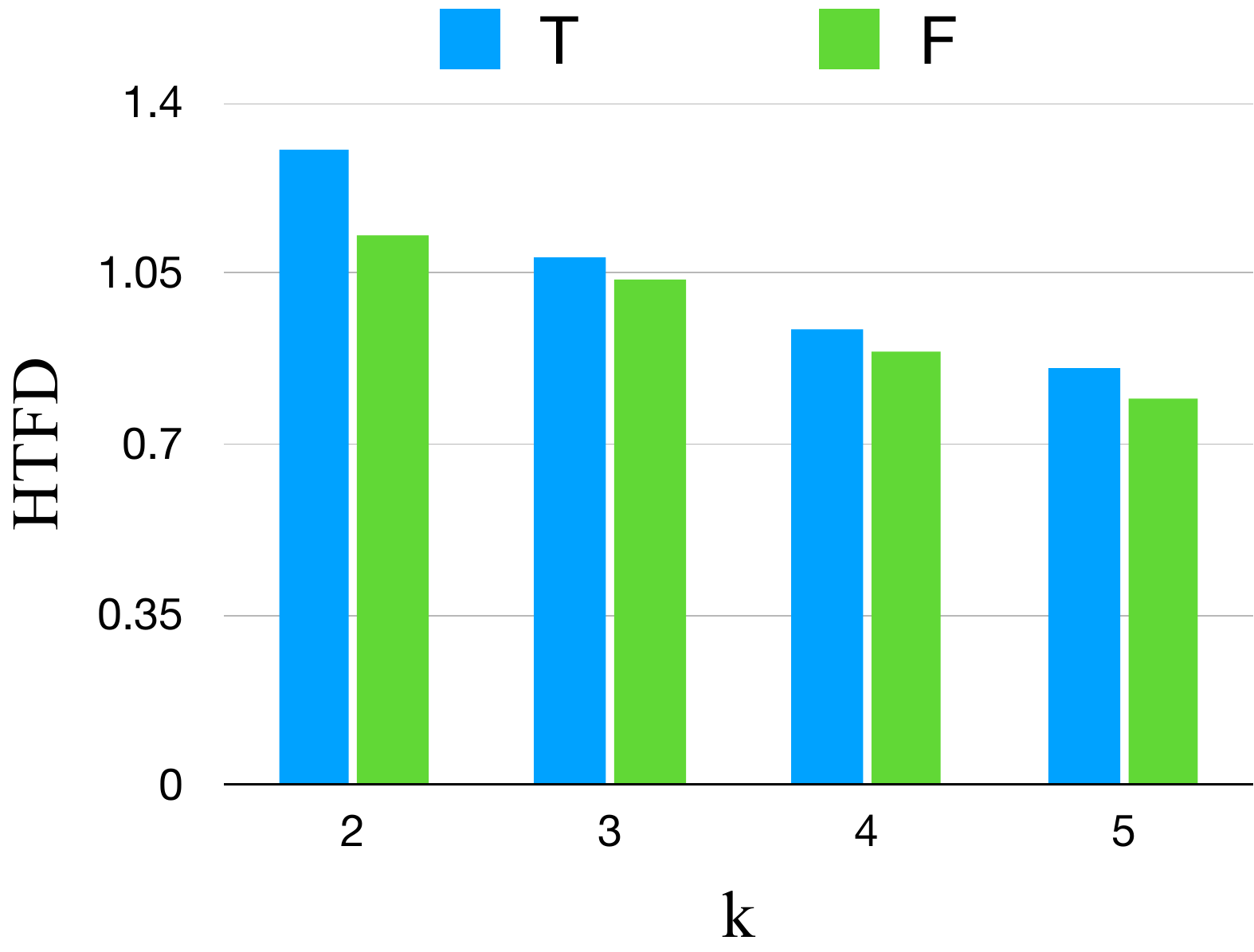}
         \caption{HTFD}
         \label{fig:sil_cic}
     \end{subfigure}
    \caption{Sensitivity Analysis (a--c) and Ablation Studies (d--f) on the CIC dataset}
    \label{fig:sensitivity_analysis}
\end{figure*}

\section{Experiments}
\label{sec:expts}

We evaluate the classification and clustering performance of \deepmethod\ on 6 large clinical datasets.
We analyze \deepmethod\ through an ablation study and investigate its sensitivity to model hyperparameters.
Finally, we present a case study that illustrates its utility in clinical risk modeling.


\subsection{Data}
We use 6 large clinical datasets; Table \ref{tab:datasets} lists the number of observations and features in each of them.
WID data is from the Women In Data Science challenge \cite{lee_wids} to predict patient mortality.
Diabetes data is from the UCI Repository where the task is readmission prediction.
All the remaining datasets have data from the 
\textsc{MIMIC} III ICU Database \cite{fei2020mimic}.
CIC and Sepsis have been used for mortality and sepsis prediction challenges in Physionet \cite{silva2012cic,reyna2019sepsis}.
Kidney and Respiratory data has been extracted by us for the tasks of Acute Kidney Failure and Acute Respiratory Distress Syndrome prediction.
All the above tasks are posed as binary classification problems. We also derive a multiclass dataset (CIC-LOS) from the CIC dataset where we predict Length of Stay (LOS), discretized into 3 classes (based on 3 quartiles). We consider a random 57-18-25 split to create training, validation and testing splits in each dataset.
More details of data preprocessing and feature extraction are in Appendix \ref{app:data_preprocessing}.

\begin{table}[!htb]
\footnotesize
\captionsetup{font=footnotesize}
\centering
\begin{adjustbox}{max width=0.5\textwidth}
\begin{tabular}{lrrr}
        \toprule
        \textbf{Dataset} & \#Instances  & \#Features & MAJ \\
        \midrule   
        WID Mortality \cite{lee_wids} & {$91711$} & {$241$} & $0.91$\\
        Diabetes \cite{UCL} & {$100000$} & {$9$} & $0.54$\\
        Sepsis \cite{reyna2019sepsis} & {$40328$} & {$89$} & $0.93$\\
        Kidney \cite{fei2020mimic} & {$14087$} & {$89$} & $0.51$\\
        Respiratory \cite{fei2020mimic} & {$22477$} & {$89$} & $0.92$\\
        CIC \cite{silva2012cic} & {$12000$} & {$117$} & $0.85$\\
        CIC-LOS & {$12000$} & {$116$} & 3 classes\\
        \bottomrule
    \end{tabular}
\end{adjustbox}
\caption{Summary of datasets. MAJ: proportion of majority class instances (for binary labels).}
    \label{tab:datasets}
\end{table}

\begin{table*}[!htb]
\centering
\begin{adjustbox}{max width=\textwidth}
\begin{tabular}{ccccccccccccccc}
\toprule

\multicolumn{1}{l}{} & \multicolumn{1}{l}{} &  &  & \multicolumn{1}{l}{\textbf{HTFD}} & & \multicolumn{1}{l}{} & \multicolumn{1}{l}{} & \multicolumn{1}{l}{} &  & \multicolumn{1}{l}{} & \multicolumn{1}{l}{\textbf{AUC}} & & \multicolumn{1}{l}{} & \multicolumn{1}{l}{} \\
\midrule

\textbf{Dataset} & \textbf{k} & \textbf{KMeans} & \textbf{DCN} & \textbf{IDEC} & \textbf{DMNN} & \textbf{\deepmethod} & & \multicolumn{1}{c}{\textbf{Baseline}} & \multicolumn{1}{c}{\textbf{SAE}} & \textbf{KMeans-Z} & \textbf{DCN-Z} & \textbf{IDEC-Z} & \textbf{DMNN} & \textbf{\deepmethod}\\
\midrule

\textbf{CIC} & 1 & \multicolumn{1}{c}{-} & \multicolumn{1}{c}{-} & \multicolumn{1}{c}{-} & \multicolumn{1}{c}{-} & \multicolumn{1}{c}{-} & & \multicolumn{1}{c}{$0.628$} & \multicolumn{1}{c}{$0.693$} & $0.804$ & $0.819$ & $0.817$ & $0.721$ & $\mathbf{0.835}$  \\
 & 2 & $1.125$ & NA & ${1.175}$ & $1.233$ & $\mathbf{1.302}$ &  &  &  & $0.705$ & NA & $0.787$ & $0.784$ & $\mathbf{0.851}$ \\
 & 3 & $1.072$ & NA & $1.081$ & $\mathbf{1.173}$ & $\mathbf{1.082}$ &  &  &  & $0.63$ & NA & $0.817$ & $0.799$ & $\mathbf{0.853}$ \\
 & 4 & $0.992$ & NA & $0.971$ & $\mathbf{1.107}$ & $0.934$ & & &  & $0.593$ & NA & $0.811$ & $0.808$ & $\mathbf{0.837}$ \\
\midrule

\textbf{Sepsis} & 1 & \multicolumn{1}{c}{-} & \multicolumn{1}{c}{-} & \multicolumn{1}{c}{-} & \multicolumn{1}{c}{-} & \multicolumn{1}{c}{-} & & \multicolumn{1}{c}{$0.654$} & \multicolumn{1}{c}{$0.681$} & $0.823$ & $0.772$ & $0.823$ & $0.670$ & $\mathbf{0.908}$ \\
 & 2 & $1.347$ & $1.365$ & $1.406$ & $\mathbf{1.476}$ & $1.283$ &  & & & $0.598$ & $0.679$ & $0.806$ & $0.767$ & $\mathbf{0.909}$ \\
 & 3 & $\mathbf{1.455}$ & $1.376$ & $1.301$ & $1.351$ & $1.319$ &  & & & $0.58$ & $0.725$ & $0.772$ & $0.794$ & $\mathbf{0.899}$ \\
 & 4 & ${1.357}$ & NA & $\mathbf{1.361}$ & $1.333$ & $1.301$ &  & & & $0.55$ & NA & $0.804$ & $0.829$ & $\mathbf{0.904}$ \\
\midrule

\textbf{AKI} & 1 & \multicolumn{1}{c}{-} & \multicolumn{1}{c}{-} & \multicolumn{1}{c}{-} & \multicolumn{1}{c}{-} & \multicolumn{1}{c}{-} & & \multicolumn{1}{c}{$0.584$} & \multicolumn{1}{c}{$0.552$} & $0.601$ & $0.591$ & $0.613$ & $0.617$ & $\mathbf{0.694}$ \\
 & 2 & ${1.108}$ & $1.014$ & $\mathbf{1.11}$ & $0.760$ & $1.044$ &  & & & $0.44$ & $0.601$ & $0.585$ & $0.690$ & $\mathbf{0.69}$ \\
 & 3 & $1.164$ & $1.166$ & $1.166$ & $0.803$ & $\mathbf{1.168}$ &  & & & $0.467$ & $0.582$ & $0.57$ & $\mathbf{0.689}$ & ${0.67}$ \\
 & 4 & $1.148$ & $1.14$ & ${1.141}$ & $0.767$ & $\mathbf{1.172}$ &  & & & $0.439$ & $0.601$ & $0.596$ & $\mathbf{0.691}$ & $0.646$ \\
\midrule

\textbf{ARDS} & 1 & \multicolumn{1}{c}{-} & \multicolumn{1}{c}{-} & \multicolumn{1}{c}{-} & \multicolumn{1}{c}{-} & \multicolumn{1}{c}{-} & & \multicolumn{1}{c}{$0.612$} & \multicolumn{1}{c}{$0.631$} & $0.726$ & $0.661$ & $0.722$ & $0.601$ & $\mathbf{0.768}$ \\
 & 2 & $1.25$ & $1.228$ & $1.283$ & $0.581$ & $\mathbf{1.312}$ &  & & & $0.628$ & $0.676$ & $0.638$ & $\mathbf{0.758}$ & $0.73$ \\
 & 3 & $1.17$ & NA & $1.171$ & $1.156$ & $\mathbf{1.282}$ & & &  & $0.644$ & NA & $0.666$ & $\mathbf{0.766}$ & $0.708$ \\
 & 4 & $1.154$ & NA & $1.119$ & $1.114$ & $\mathbf{1.262}$ & & &  & $0.616$ & NA & $0.635$ & $0.697$ & $\mathbf{0.744}$ \\
\midrule

\textbf{WID-M} & 1 & \multicolumn{1}{c}{-} & \multicolumn{1}{c}{-} & \multicolumn{1}{c}{-} & \multicolumn{1}{c}{-} & \multicolumn{1}{c}{-} & & \multicolumn{1}{c}{$0.711$} & \multicolumn{1}{c}{$0.703$} & $0.853$ & $0.814$ & $0.853$ & $0.686$ & $\mathbf{0.883}$ \\
 & 2 & $1.285$ & $1.23$ & $\mathbf{1.571}$ & $1.438$ & $1.534$ &  & & & $0.731$ & $0.739$ & $0.8$ & $0.846$ & $\mathbf{0.883}$ \\
 & 3 & $1.328$ & NA & $1.501$ & $1.404$ & $\mathbf{1.514}$ & & &  & $0.709$ & NA & $0.809$ & $0.853$ & $\mathbf{0.874}$ \\
 & 4 & $1.41$ & NA & $1.388$ & $1.363$ & $\mathbf{1.469}$ & & &  & $0.691$ & NA & $0.806$ & $\mathbf{0.865}$ & ${0.859}$ \\
\midrule

\textbf{Diabetes} & 1 & \multicolumn{1}{c}{-} & \multicolumn{1}{c}{-} & \multicolumn{1}{c}{-} & \multicolumn{1}{c}{-} & \multicolumn{1}{c}{-} & & \multicolumn{1}{c}{$0.566$} & \multicolumn{1}{c}{$0.551$} & $0.58$ & $0.574$ & $0.582$ & ${0.578}$ & $\mathbf{0.587}$ \\
 & 2 & $1.633$ & ${1.554}$ & ${1.707}$ & $1.663$ & $\mathbf{1.802}$ &  & & & $0.53$ & $0.566$ & $0.567$ & ${0.581}$ & $\mathbf{0.587}$ \\
 & 3 & $1.593$ & $\mathbf{1.53}$ & $\mathbf{1.698}$ & $1.245$ & $1.67$ & & &  & $0.514$ & $0.568$ & $0.563$ & ${0.579}$ & $\mathbf{0.585}$ \\
 & 4 & $\mathbf{1.611}$ & $1.398$ & $1.578$ & $1.381$ & ${1.609}$ &  & & & $0.507$ & $0.569$ & $0.563$ & ${0.581}$ & $\mathbf{0.586}$ \\
\midrule

\textbf{CIC-LOS} & 1 & \multicolumn{1}{c}{-} & \multicolumn{1}{c}{-} & \multicolumn{1}{c}{-} & \multicolumn{1}{c}{-} & \multicolumn{1}{c}{-} & & \multicolumn{1}{c}{$0.646$} & \multicolumn{1}{c}{$0.659$} & $0.652$ & $0.639$ & $0.642$ & \multicolumn{1}{c}{-} & $\mathbf{0.663}$ \\
 & 2 & $1.042$ & $1.003$ & $\mathbf{1.148}$ & \multicolumn{1}{c}{-} & $1.101$ & & & & $0.575$ & $0.632$ & $0.648$ & \multicolumn{1}{c}{-} & $\mathbf{0.672}$ \\
 & 3 & $1.038$ & $0.99$ & $\mathbf{1.052}$ & \multicolumn{1}{c}{-} & ${0.925}$ & & & & $0.567$ & $0.636$ & $0.642$ & \multicolumn{1}{c}{-} & $\mathbf{0.664}$\\
 & 4 & $\mathbf{0.995}$ & NA & $0.955$ & \multicolumn{1}{c}{-} & ${0.871}$ &  & & & $0.539$ & NA & $0.64$ & \multicolumn{1}{c}{-} & $\mathbf{0.664}$\\
\bottomrule

\end{tabular}
\end{adjustbox}
\caption{Clustering and Classification performance. Best values in bold. NA indicates no results due to an empty cluster. Note that $\operatorname{HTFD}$ values are not defined for $k=1$.}
\label{tab:clustering}
\end{table*}

\subsection{Classification and Clustering Performance}
\subsubsection{Baselines}
As baselines for the classification task, we use 3 different kinds of techniques.
The first is a feedforward neural network ({\bf baseline}) that only performs classification.
It's architecture is set to be identical to that of a combination of \deepmethod's encoder and a single local network.
The second is the Supervised Autoencoder ({\bf SAE}) \cite{leli2018supervised} that uses reconstruction loss as an unsupervised regularizer (similar architecture as \textbf{baseline}).
The third category uses the common `cluster-then-predict' approach, where clustering is independently performed first and classifiers are trained on each cluster (denoted by {\bf -Z}).
In this category, we compare with
$3$ clustering methods \textbf{$k$-means} (where we use an autoencoder to get embeddings which are then clustered using $k$-means), Deep Clustering Network {\bf DCN} \cite{yang2017towards} and Improved Deep Embedded Clustering ({\bf IDEC}) \cite{Guo2017IDEC}. 
We also use Deep Mixture Neural Networks ({\bf DMNN}) \cite{li2020dmnn} that follows this paradigm.
Note that their implementation only supports binary classification.
We use stochastic cohort sampling during prediction in all baselines that use clustering for a fair comparison.
We compare the clustering performance 
obtained by \deepmethod\ 
with that of DMNN, $k$-means, DCN and IDEC.




\subsubsection{Performance Metrics}
Standard classification metrics, Area Under the Receiver Operating Characteristic (AUC) and F1 scores, are used for binary classification. For multiclass setting, we use one-vs-rest algorithm to calculate AUC.

To evaluate clustering we use Silhouette scores and another score, HTFD, described below, that evaluates feature discrimination in the inferred clusters.
A common practice in clustering studies (e.g., \cite{li2015identification}) is to check, for each feature, if 
there is a statistically significant difference in
its distribution across the clusters.
A low p-value ($< 0.05$) indicates significant difference.
To aggregate this over all features $F$, we define, for each cluster in a clustering, the metric
\textbf{H}ypothesis 
\textbf{T}esting 
based \textbf{F}eature \textbf{D}iscrimination ($\operatorname{HTFD}$):

\[
\operatorname{HTFD}(C_i) = \frac{1}{|F|} \sum_{f\in F} -\ln\left(\text{p-value}(X_i^f, X^f)\right)*0.05
\]
where $X_i^f$ denotes the values of feature $f$ for data points in cluster $C_i$ and $X^f$ denotes the 
feature values of data points in all clusters except $C_i$.
Student's t-test is used to obtain the p-value.
The negative logarithm of p-value is multiplied by the significance level, to normalize and obtain a measure wherein higher values indicate better clustering. To obtain an overall value for the entire clustering, we take the average of individual, cluster-wise  HTFD values.



\subsubsection{Hyperparameter Settings}

For \textbf{DCN-Z} and \textbf{IDEC-Z}, we use default parameters as suggested by their authors. For \deepmethod\ we let $(\beta, \gamma, \delta) = (0.5, 1.5, 1)$. We select these default values by considering the sensitivity analysis (see Fig. \ref{fig:sensitivity_analysis}). 
The common encoder has layers of size 128-64-32 and local expert network has layers of size 64-32-16-8. The size of latent embeddings is $20$. 
We evaluate all the methods for four different values of $K = 1,2,3,4$.

\subsubsection{Results}
Table \ref{tab:clustering} shows the performance of all methods compared, on classification and clustering.
We observe that cluster-then-predict methods (\textbf{*-Z}) and DMNN generally perform better than SAE, which suggests that clustering-based regularization is more effective than reconstruction-based regularization.
SAE outperforms the baseline on most datasets.
These results align with previous research that show that unsupervised regularization aids supervised learning.
\deepmethod\ outperforms all methods on all datasets  for at least one input cluster size. 
Also note that \deepmethod's performance does not degrade with increasing number of clusters as seen for KMeans-Z.
Overall, for classification, \deepmethod\ consistently outperforms all the baselines on all datasets.

With respect to clustering, the performance of \deepmethod\ is superior to the best baseline in 5 datasets, for some values of $k$.
For other values of $k$, and on other datasets, the performance values are comparable or lower.
F1 scores and Silhouette scores are in Appendix \ref{app:extended_results}, where the performance trends are similar.
Overall, the results show that \deepmethod\ achieves clustering performance that is comparable to, and 
achieves classification performance that is considerably better than, the state-of-the-art alternatives respectively.


\subsection{Sensitivity Analysis and Ablation Studies}

We evaluate the effect of hyperparameters $\beta, \gamma$ and $\delta$ on \deepmethod.
We individually vary each hyperparameter while setting other two to $0$ and measure the classification performance.
Figures \ref{fig:f1_auc_beta}, \ref{fig:f1_auc_gamma} and  \ref{fig:f1_auc_delta} show the results for the CIC dataset. We observe that both the F1 and AUC values are fairly robust to changes in their values.

Our stochastic cohort sampling approach may or may not be used independently during training and prediction.
To evaluate it's effect we evaluate all four combinations through an ablation study.
We denote the combinations by TT, TF, FT and FF, where the first and second positions indicate training and prediction respectively.
T indicates use of our approach while F indicates that it is not used.
Figure \ref{fig:f1_cic} and \ref{fig:auc_cic} show the F1 and AUC scores for all 4 combinations on the CIC dataset (results on Sepsis dataset in Appendix \ref{app:extended_ablation_analysis}). The best performance is achieved when the approach is used both in training and prediction (TT).
For clustering, there is no prediction. 
Figure \ref{fig:sil_cic} shows that the performance is better with rather than without sampling.
All the results shown are averages over $5$ runs. 

\subsection{Case Study: Mortality Prediction}
\label{sec:case_study}
As a case study, we illustrate the use of \deepmethod\ on the CIC mortality prediction data for $k=3$ clusters.
Since clustering is done simultaneously with classification, the clusters are influenced by the target label, i.e., mortality indicator, and thus, by design, \deepmethod\ is expected to find mortality subtypes.
In other words, we expect the inferred clusters to have different risk factors tailored to each underlying subpopulation.
To evaluate this, we examine feature importances for each cluster's local risk model.
We distil the knowledge of the local networks 
into a simpler student model \cite{gou2021knowledge},
a Gradient Boosting Classifier, in our case,
that provides feature importance values.

\begin{table}
\centering
\footnotesize
\captionsetup{font=footnotesize}
\begin{tabular}{@{}lll@{}}
\toprule
$|C_1|=2458$ & $|C_2|=1890$ & $|C_3|=2402$ \\
$\%D = 0.115$ & $\%D = 0.032$ & $\%D = 0.254$ \\ \midrule
\cellcolor{yellow!25}{Age} & \cellcolor{yellow!25}{Age} & \cellcolor{yellow!25}{GCS\_last} \\
\cellcolor{yellow!25}{GCS\_last} & {UrineOutputSum} & \cellcolor{yellow!25}{BUN\_last} \\
\cellcolor{yellow!25}{BUN\_last} & \cellcolor{yellow!25}{GCS\_last} & \cellcolor{yellow!25}Age \\
\cellcolor{blue!25}{RespRate\_median} & \cellcolor{blue!25}{BUN\_first} & Lactate\_last \\
\cellcolor{blue!25}{BUN\_first} & CSRU & Bilirubin\_last \\
Weight\_last & GCS\_lowest & Length\_of\_stay \\
\cellcolor{blue!25}{HR\_highest} & MechVentDuration & GCS\_median \\
Weight\_first & \cellcolor{yellow!25}{BUN\_last} & SOFA \\
\cellcolor{blue!25}{Weight} & \cellcolor{blue!25}{RespRate\_median} & \cellcolor{blue!25}{Weight} \\
RespRate\_highest & \cellcolor{blue!25}{HR\_highest} & HCO3\_last \\ \bottomrule
\end{tabular}
\caption{Important Features for mortality prediction in respective clusters. Features common in all three clusters are highlighted in yellow while those common in two are highlighted in purple. Rest are unique to their respective clusters.}
\label{tab:important_features}
\end{table}

In Table \ref{tab:important_features}, we list the cluster sizes $C_i$, proportion of patients who do not survive ($\%D$) and the top 10 most important features in each cluster.
We see that, as expected, the risk models created by \deepmethod\ deem different sets of features as important 
for predicting mortality. Out of the $10$ most important features, each risk model contains $3$ common features, $4$, $3$ and $1$ features common across two models respectively and $3$, $4$ and $6$ features unique to respective clusters.
Appendix \ref{app:cic_clusters} presents further analysis of the clusters.

\section{Conclusion}
We design  \deepmethod\, a model for combined clustering and classification and theoretically analyze its generalization properties. 
Our model can be viewed as a mixture of expert networks trained on inferred clusters in the data.
Our experiments show that both regularization through unsupervised clustering and stochastic sampling strategy during training lead to substantial improvement in classification performance.
We demonstrate the efficacy of \deepmethod\ on several clinical datasets, for predictive modeling tailored to subpopulations inferred from the data.

The clustering performance of \deepmethod\ is comparable to that of other deep clustering methods  and may be improved further.
This can be explored in future work along with ways to combine hierarchical clustering algorithms with supervised models and applications in other domains.

\bibliographystyle{named}
\bibliography{ijcai22}

\appendix
\onecolumn

\section{Analysis} \label{app:1}
\subsection{General case}

The following lemma is from \cite[Theorem 3.1]{bartlett2002rademacher,mohri2012foundations}: 
\begin{lemma} \label{lemma:1}
Let $\Fcal$  be a  set of maps $x\mapsto f(x)$. Suppose that $0 \le \ell\left(q, y\right) \le \lambda$ for any $q \in\{f(x):f\in \Fcal, x\in \Xcal \}$ and  $y \in \Ycal$. Then, for any $\delta>0$, with probability at least $1-\delta$ (over  an i.i.d. draw of $N$ i.i.d.  samples  $((x_i, y_i))_{i=1}^N$), the following holds: for all maps $ f\in\Fcal$,
\begin{align} 
\EE_{x,y}[\ell(f(x),y)]
\le \frac{1}{N}\sum_{i=1}^{N} \ell(f(x_i),y_i)+2\hat \Rcal_{N}(\ell \circ \Fcal)+3\lambda \sqrt{\frac{\ln(1/\delta)}{2N}},   
\end{align}
where  $\hat \Rcal_{N}(\ell \circ \Fcal):=\EE_{\sigma}[\sup_{f \in \Fcal}\frac{1}{N} \sum_{i=1}^N \sigma_i \ell(f(x_{i}),y_{i})]$ and $\sigma_1,\dots,\sigma_N$ are independent uniform random variables taking values in $\{-1,1\}$.  
\end{lemma}

We can use Lemma \ref{lemma:1} to prove the following:

\begin{theorem} \label{thm:1}
Let $\Fcal_j$ be a  set of maps $x \in \Omega_j \mapsto f_{j}(x)$.  Let $\Fcal= \{x\mapsto f(x):f(x) = \sum_{j=1}^K \one\{x \in \Omega_j\} f_j(x), f_j \in \Fcal_j\}$. Suppose that $0 \le \ell\left(q, y\right) \le \lambda_{j}$ for any $q \in\{f(x):f\in \Fcal_{j}, x\in \Omega_j\}$ and  $y \in \Ycal_{j}$.  Then, for any $\delta>0$, with probability at least $1-\delta$, the following holds: for all maps $ f \in\Fcal$,
\begin{align} 
&\EE_{x,y}[\ell(f(x),y)]
\le \sum _{j=1}^k\Pr(x \in \Omega_j) \left(\frac{1}{N_{j}}\sum_{i=1}^{N_{j}} \ell(f_{}(x_i^{j}),y_i^{j})+2\hat \Rcal_{N_{j}}(\ell \circ \Fcal_{j})+3\lambda_{j} \sqrt{\frac{\ln(k/\delta)}{2N_{j}}}\right),   
\end{align}
where  $\hat \Rcal_{N}(\ell \circ \Fcal_{j}):=\EE_{\sigma}[\sup_{f_{j} \in \Fcal_{j}}\frac{1}{N} \sum_{i=1}^N \sigma_i \ell(f_{j}(x_{i}),y_{i})]$ and $\sigma_1,\dots,\sigma_N$ are independent uniform random variables taking values in $\{-1,1\}$.  
\end{theorem}
\begin{proof}[Proof of Theorem \ref{thm:1}]
We have that
\begin{align} \label{eq:1}
\EE_{x,y}[\ell(f(x),y)] &=\sum _{j=1}^k\Pr(x \in \Omega_j)\EE_{x,y}[\ell(f(x),y) \mid x  \in \Omega_j)]
\\ \nonumber & =\sum _{j=1}^k\Pr(x \in \Omega_j)\EE_{x,y}[\ell(f_{j}(x),y) \mid x  \in \Omega_j)].  
\end{align}
Since the conditional probability distribution is a probability distribution, we apply Lemma \ref{lemma:1} to each term and take union bound to obtain the following: for any $\delta>0$, with probability at least $1-\delta$, for all $j \in \{1,\dots,k\}$ and all $f_j \in F_j$,
\begin{align*} 
\EE_{x,y}[\ell(f_{j}(x),y) \mid x  \in \Omega_j)] &\le\frac{1}{N_{j}}\sum_{i=1}^{N_{j}} \ell(f_{j}(x_i^{j}),y_i^{j})+2\hat \Rcal_{N_{j}}(\ell \circ \Fcal_{j})+3\lambda_{j} \sqrt{\frac{\ln(k/\delta)}{2N}}
\\ & =\frac{1}{N_{j}}\sum_{i=1}^{N_{j}} \ell(f_{}(x_i^{j}),y_i^{j})+2\hat \Rcal_{N_{j}}(\ell \circ \Fcal_{j})+3\lambda_{j} \sqrt{\frac{\ln(k/\delta)}{2N}}.
\end{align*}
Thus, using \eqref{eq:1}, we sum up both sides with the factors $\Pr(x \in \Omega_j)$ to yield:

\begin{align}
\EE_{x,y}[\ell(f(x),y)] & =\sum _{j=1}^k\Pr(x \in \Omega_j)\EE_{x,y}[\ell(f_{j}(x),y) \mid x  \in \Omega_j)]
\\ & \le \sum _{j=1}^k\Pr(x \in \Omega_j) \left(\frac{1}{N_{j}}\sum_{i=1}^{N_{j}} \ell(f_{}(x_i^{j}),y_i^{j})+2\hat \Rcal_{N_{j}}(\ell \circ \Fcal_{j})+3\lambda_{j} \sqrt{\frac{\ln(k/\delta)}{2N}}\right).
\end{align}
\end{proof}

\subsection{Special case: multi-class classification}

For multi-class classification problems with $\Bcal$ classes and  $y \in \{1,\dots,\Bcal\}$, define the margin loss as follows: 
 $$
 {\ell}_{\rho}(f(x),y)={\ell}_{\rho}^{(2)} ({\ell}_{\rho}^{(1)}(f(x),y))
 $$  
 where 
$$
{\ell}_{\rho}^{(1)}(f(x),y)=f(x)_{y}- \max_{y \neq y'} f(x)_{y'}\in \RR,
$$
and 
$$
{\ell}_{\rho}^{(2)}(q)=
\begin{cases} 0 & \text{if } \rho \le q \\
1-q/\rho &  \text{if } 0 \le q \le \rho \\
1 & \text{if } q \le 0.\\
\end{cases}
$$

Define the 0-1 loss  as:
$$
\textstyle \ell_{01}(f(x),y) = \one \{y \neq \argmax_{y'\in[T]} f(x)_{y'}\}, 
$$
where $[T] = \{1,\dots,\Bcal\}$. For any $\rho>0$, the margin loss ${\ell}_{\rho}(f(x),y)$ is an upper bound on the 0-1 loss: i.e., ${\ell}_{\rho}(f(x),y) \ge \ell_{01}(f(x),y)$.

The following lemma is from \cite[Theorem 8.1]{mohri2012foundations}:
\begin{lemma} \label{lemma:2}
Let $\Fcal$  be a  set of maps $x\mapsto f(x)$. Fix $\rho>0$. Then, for any $\delta>0$, with probability at least $1-\delta$ over  an i.i.d. draw of $N$ i.i.d. test samples  $((x_i, y_i))_{i=1}^m$, the following holds: for all maps $ f\in\Fcal$,
\begin{align*}
\EE_{x,y}[\ell_{01}(f(x),y)]\le \frac{1}{N}\sum_{i=1}^{N} \ell_{\rho}(f(x_i),y_i)+\frac{2\Bcal^2}{\rho} \hat \Rcal_{N}(\Fcal_\Bcal)+\left(1+\frac{2\Bcal^2}{\rho}\right) \sqrt{\frac{\ln(2/\delta)}{2N}},   
\end{align*}
where  $\Fcal_\Bcal=\{x\mapsto f(x)_k : f \in\Fcal, k \in [T]\}$.
\end{lemma}
\begin{proof} Using Theorem 8.1 \cite{mohri2012foundations}, we have that with probability at least $1-\delta/2$, 
\begin{align*}
\EE_{x,y}[\ell_{01}(f(x),y)]\le \frac{1}{N}\sum_{i=1}^{N} \ell_{\rho}(f(x_i),y_i)+\frac{2\Bcal^2}{\rho} \Rcal_{N}( \Fcal_\Bcal)+ \sqrt{\frac{\ln(2/\delta)}{2N}},   
\end{align*}
Since changing one point in $S$ changes  $\hat \Rcal_{m}(\Fcal_\Bcal)$ by at most $1/N$, McDiarmid's inequality implies the  statement of this lemma by taking union bound.
\end{proof}

\begin{theorem} \label{thm:2}
Let $\Fcal_j$ be a  set of maps $x \in \Omega_j \mapsto f_{j}(x)$.  Let $\Fcal= \{x\mapsto f(x):f(x) = \sum_{j=1}^k \one\{x \in \Omega_j\} f_j(x), f_j \in \Fcal_j\}$. Suppose that $0 \le \ell\left(q, y\right) \le \lambda_{j}$ for any $q \in\{f(x):f\in \Fcal_{j}, x\in \Omega_j\}$ and  $y \in \Ycal_{j}$.  Then, for any $\delta>0$, with probability at least $1-\delta$, the following holds: for all maps $ f \in\Fcal$,
\begin{align} 
&\EE_{x,y}[\ell_{01}(f(x),y)]
\le \sum _{j=1}^k\Pr(x \in \Omega_j) \left(\frac{1}{N_{j}}\sum_{i=1}^{N_{j}} \ell_{\rho}(f_{}(x_i^{j}),y_i^{j})+\frac{2\Bcal^2}{\rho}\hat \Rcal_{N_{j}}( \Fcal_{T,j})+\left(1+\frac{2\Bcal^2}{\rho}\right) \sqrt{\frac{\ln(2k/\delta)}{2N_{j}}}\right),   
\end{align}
where  $\hat \Rcal_{N_j}( \Fcal_{T,j}):=\EE_{\sigma}[\sup_{f_{j} \in \Fcal_{T,j}}\frac{1}{N} \sum_{i=1}^{N_j} \sigma_i \ell(f_{j}(x_i^{j}),y_i^{j})]$ and $\sigma_1,\dots,\sigma_N$ are independent uniform random variables taking values in $\{-1,1\}$. Here,   $\Fcal_{\Bcal,j}=\{x\mapsto f(x)_k : f \in\Fcal_{j}, k \in [\Bcal]\}$.
  
\end{theorem}
\begin{proof}[Proof of Theorem \ref{thm:1}]
We have that
\begin{align} \label{eq:2}
\EE_{x,y}[\ell_{01}(f(x),y)] &=\sum _{j=1}^K\Pr(x \in \Omega_j)\EE_{x,y}[\ell_{01}(f(x),y) \mid x  \in \Omega_j)]
\\ \nonumber & =\sum _{j=1}^k\Pr(x \in \Omega_j)\EE_{x,y}[\ell_{01}(f_{j}(x),y) \mid x  \in \Omega_j)].  
\end{align}
Since the conditional probability distribution is a probability distribution, we apply Lemma \ref{lemma:1} to each term and take union bound to obtain the following: for any $\delta>0$, with probability at least $1-\delta$, for all $j \in \{1,\dots,k\}$ and all $f_j \in F_j$,
\begin{align*} 
\EE_{x,y}[\ell_{01}(f_{j}(x),y) \mid x  \in \Omega_j)] &\le\frac{1}{N_{j}}\sum_{i=1}^{N_{j}} \ell_{\rho}(f_{j}(x_i^{j}),y_i^{j})+\frac{2\Bcal^2}{\rho}\hat \Rcal_{N_{j}}( \Fcal_{\Bcal,j})+\left(1+\frac{2\Bcal^2}{\rho}\right) \sqrt{\frac{\ln(2k/\delta)}{2N}}
\\ & =\frac{1}{N_{j}}\sum_{i=1}^{N_{j}} \ell_{\rho}(f_{}(x_i^{j}),y_i^{j})+\frac{2\Bcal^2}{\rho}\hat \Rcal_{N_{j}}( \Fcal_{\Bcal,j})+\left(1+\frac{2\Bcal^2}{\rho}\right) \sqrt{\frac{\ln(2k/\delta)}{2N}}.
\end{align*}
Thus, using \eqref{eq:2}, we sum up both sides with the factors $\Pr(x \in \Omega_j)$ to yield:

\begin{align}
\EE_{x,y}[\ell_{01}(f(x),y)] & =\sum _{j=1}^k\Pr(x \in \Omega_j)\EE_{x,y}[\ell_{01}(f_{j}(x),y) \mid x  \in \Omega_j)]
\\ & \le \sum _{j=1}^k\Pr(x \in \Omega_j) \left(\frac{1}{N_{j}}\sum_{i=1}^{N_{j}} \ell_{\rho}(f_{}(x_i^{j}),y_i^{j})+\frac{2\Bcal^2}{\rho}\hat \Rcal_{N_{j}}( \Fcal_{\Bcal,j})+\left(1+\frac{2\Bcal^2}{\rho}\right)  \sqrt{\frac{\ln(2k/\delta)}{2N}}\right). 
\end{align}
\end{proof}

\subsection{Special case: deep neural networks with binary classification}

Note that for any $\rho>0$, the margin loss ${\ell}_{\rho}(f(x),y)$ is an upper bound on the 0-1 loss: i.e., ${\ell}_{\rho}(f(x),y) \ge \ell_{01}(f(x),y)$.

The following lemma is from \cite[Theorem 4.4]{mohri2012foundations}:
\begin{lemma} \label{lemma:3}
Let $\Fcal$  be a  set of real-valued functions. Fix $\rho>0$. Then, for any $\delta>0$, with probability at least $1-\delta$ over  an i.i.d. draw of $m$ i.i.d. test samples  $((x_i, y_i))_{i=1}^m$, the following holds: for all maps $ f\in\Fcal$,
\begin{align*}
\EE_{x,y}[\ell_{01}(f(x),y)]\le \frac{1}{m}\sum_{i=1}^{m} \ell_{\rho}(f(x_i),y_i)+\frac{2}{\rho} \hat \Rcal_{m}(\Fcal)+3 \sqrt{\frac{\ln(2/\delta)}{2m}}. 
\end{align*}
\end{lemma}
\begin{theorem} \label{thm:4}
Let $\Fcal_j$ be a  set of maps $x \in \Omega_j \mapsto f_{j}(x)$.  Let $\Fcal= \{x\mapsto f(x):f(x) = \sum_{j=1}^K \one\{x \in \Omega_j\} f_j(x), f_j \in \Fcal_j\}$. Suppose that $0 \le \ell\left(q, y\right) \le \lambda_{j}$ for any $q \in\{f(x):f\in \Fcal_{j}, x\in \Omega_j\}$ and  $y \in \Ycal_{j}$.  Then, for any $\delta>0$, with probability at least $1-\delta$, the following holds: for all maps $ f \in\Fcal$,
\begin{align} 
&\EE_{x,y}[\ell_{01}(f(x),y)]
\le\sum _{j=1}^k\Pr(x \in \Omega_j) \left(\frac{1}{N_{j}}\sum_{i=1}^{N_{j}} \ell_{\rho}(f_{}(x_i^{j}),y_i^{j})+\frac{2}{\rho}\hat \Rcal_{N_{j}}( \Fcal_{j})+3 \sqrt{\frac{\ln(2/\delta)}{2N_{j}}}\right),   
\end{align}
where  $\hat \Rcal_{N_j}(\Fcal_{j}):=\EE_{\sigma}[\sup_{f_{j} \in \Fcal_{j}}\frac{1}{N} \sum_{i=1}^{N_j} \sigma_i \ell(f_{j}(x_i^{j}),y_i^{j})]$ and $\sigma_1,\dots,\sigma_N$ are independent uniform random variables taking values in $\{-1,1\}$.  
\end{theorem}
\begin{proof}[Proof of Theorem \ref{thm:4}]
We have that
\begin{align} \label{eq:3}
\EE_{x,y}[\ell_{01}(f(x),y)] &=\sum _{j=1}^k\Pr(x \in \Omega_j)\EE_{x,y}[\ell_{01}(f(x),y) \mid x  \in \Omega_j)]
\\ \nonumber & =\sum _{j=1}^K\Pr(x \in \Omega_j)\EE_{x,y}[\ell_{01}(f_{j}(x),y) \mid x  \in \Omega_j)].  
\end{align}
Since the conditional probability distribution is a probability distribution, we apply Lemma \ref{lemma:3} to each term and take union bound to obtain the following: for any $\delta>0$, with probability at least $1-\delta$, for all $j \in \{1,\dots,k\}$ and all $f_j \in F_j$,
\begin{align*} 
\EE_{x,y}[\ell_{01}(f_{j}(x),y) \mid x  \in \Omega_j)] &\le\frac{1}{N_{j}}\sum_{i=1}^{N_{j}} \ell_{\rho}(f_{j}(x_i^{j}),y_i^{j})+\frac{2}{\rho}\hat \Rcal_{N_{j}}( \Fcal_{j})+3 \sqrt{\frac{\ln(2/\delta)}{2N_{j}}}
\\ & =\frac{1}{N_{j}}\sum_{i=1}^{N_{j}} \ell_{\rho}(f_{}(x_i^{j}),y_i^{j})+\frac{2}{\rho}\hat \Rcal_{N_{j}}( \Fcal_{j})+3 \sqrt{\frac{\ln(2/\delta)}{2N_{j}}}. 
\end{align*}
Thus, using \eqref{eq:3}, we sum up both sides with the factors $\Pr(x \in \Omega_j)$ to yield:

\begin{align}
\EE_{x,y}[\ell_{01}(f(x),y)] & =\sum _{j=1}^K\Pr(x \in \Omega_j)\EE_{x,y}[\ell_{01}(f_{j}(x),y) \mid x  \in \Omega_j)]
\\ & \le \sum _{j=1}^k\Pr(x \in \Omega_j) \left(\frac{1}{N_{j}}\sum_{i=1}^{N_{j}} \ell_{\rho}(f_{}(x_i^{j}),y_i^{j})+\frac{2}{\rho}\hat \Rcal_{N_{j}}( \Fcal_{j})++3 \sqrt{\frac{\ln(2/\delta)}{2N_{j}}}\right). 
\end{align}
\end{proof}

We are now ready to prove Theorem \ref{thm:5}.

\begin{proof}[Proof of Theorem \ref{thm:5}]
Using Theorem \ref{thm:4} with  $\Pr(x \in \Omega_j)=1/k$ and $m_{j}=N/k$, for $f\in\Fcal^{k}$,

\begin{align*}
\EE_{x,y}[\ell_{01}(f(x),y)] &\le \frac{1}{k}\sum _{j=1}^k \left(\frac{1}{(N/k)}\sum_{i=1}^{N/k} \ell_{\rho}(f(x_i^{j}),y_i^{j})+2 \rho^{-1}  \hat \Rcal_{N/k}( \Fcal_{j})+3 \sqrt{\frac{\ln(2/\delta)}{2(N/k)}}\right)   
\\ & =\frac{1}{k}\sum _{j=1}^k \left(\frac{k}{N}\sum_{i=1}^{N/k} \ell_{\rho}(f_{}(x_i^{j}),y_i^{j})+2 \rho^{-1}  \hat \Rcal_{N/k}( \Fcal_{j})+3 \sqrt{\frac{k\ln(2k/\delta)}{2N}}\right)
\\ & =\frac{1}{N}\sum _{j=1}^k \sum_{i=1}^{N/k} \ell_{\rho}(f_{}(x_i^{j}),y_i^{j})+\frac{2 \rho^{-1}}{k}\sum _{j=1}^k  \hat \Rcal_{N/k}( \Fcal_{j})+3\sqrt{\frac{k\ln(2k/\delta)}{2N}}
\\ & = \hat \EE_{x,y}[\ell_{\rho}(f(x),y)] +\frac{2 \rho^{-1}}{k}\sum _{j=1}^k  \hat \Rcal_{N/k}( \Fcal_{j})+3\sqrt{\frac{k\ln(2k/\delta)}{2N}}
\end{align*}
Here, using Theorem 1 of \cite{golowich2018size}, we have that 
$$
\hat \Rcal_{N/k}( \Fcal_{j})\le \frac{B_j \sqrt{k}(\sqrt{2 \log(2) (L+Q) }+1)(\prod_{l=1}^L M_l^{j})(\prod_{l=1}^Q M_l^{0})}{\sqrt{N}}.
$$
Thus,  for any $\delta>0$, with probability at least $1-\delta$ over  an i.i.d. draw of $N$ i.i.d. test samples  $((x_i, y_i))_{i=1}^N$, the following holds: for all maps $ f\in\Fcal^{K}$,

\begin{align*}
&\EE_{x,y}[\ell_{01}(f(x),y)] - \hat \EE_{x,y}[\ell_{\rho}(f(x),y)] 
\\ &\le \frac{2 \rho^{-1}}{K}\sum _{j=1}^K  \frac{B_j \sqrt{k}(\sqrt{2 \log(2) (L+Q) }+1)(\prod_{l=1}^L M_l^{j})(\prod_{l=1}^Q M_l^{0})}{\sqrt{N}}+3\sqrt{\frac{k\ln(2k/\delta)}{2N}} \\ &\le\frac{2 \rho^{-1}(\sqrt{2 \log(2) (L+Q) }+1)(\prod_{l=1}^Q M_l^{0})\left(\sum _{j=1}^k B_j \prod_{l=1}^L M_l^{j} \right)}{\sqrt{kN}} +3\sqrt{\frac{k\ln(2k/\delta)}{2N}}
\end{align*}
\end{proof}

By taking union bound over  $M_l^j$, it is straightforward to replace $\prod_{l=1}^L M_l^{j}$ by $\prod_{l=1}^L \|W_{l}^{j}\|_F$ in the bound of Theorem \ref{thm:5} with additional factors   that is not dominant and  hidden in the $\tilde O$ notation (see the proof of Lemma A.9 of \cite{bartlett2017spectrally} or the proof of Theorem 2 of \cite{koltchinskii2002empirical}).
\newpage

\section{Extended Results}
\label{app:extended_results}
\begin{table*}[!hb]
\centering
\begin{adjustbox}{max width=\textwidth}
\begin{tabular}{ccccccc}
\toprule
 &  &  &  & \textbf{SIL} & \multicolumn{1}{c}{} &  \\
 \midrule
 
\multicolumn{1}{c}{\textbf{Dataset}} & \multicolumn{1}{c}{\textbf{k}} & \multicolumn{1}{c}{\textbf{KMeans}} & \multicolumn{1}{c}{\textbf{DCN}} & \textbf{IDEC} & \textbf{DMNN} & \multicolumn{1}{c}{\textbf{\deepmethod}} \\
\midrule

\multicolumn{1}{c}{\textbf{CIC}} & 2 & \multicolumn{1}{c}{\textbf{0.54}} & NA & \multicolumn{1}{c}{0.51} & 0.125 & 0.358 \\
 & 3 & 0.297 & NA & \textbf{0.354} & 0.092 & 0.305 \\
 & 4 & 0.168 & NA & \textbf{0.326} & 0.085 & 0.231 \\
\midrule

\multicolumn{1}{c}{\textbf{Sepsis}} & 2 & \multicolumn{1}{c}{\textbf{0.588}} & 0.253 & \multicolumn{1}{c}{0.325} & 0.250 & 0.357 \\
 & 3 & 0.277 & 0.242 & \textbf{0.406} & 0.104 & 0.349 \\
 & 4 & 0.212 & NA & \textbf{0.307} & 0.099 & 0.248 \\
\midrule

\multicolumn{1}{c}{\textbf{AKI}} & 2 & 0.401 & 0.596 & \textbf{0.769} & 0.259 & 0.638 \\
 & 3 & 0.459 & 0.782 & \textbf{0.828} & 0.285 & 0.611 \\
 & 4 & 0.459 & 0.585 & \textbf{0.622} & 0.278 & 0.53 \\
\midrule

\multicolumn{1}{c}{\textbf{ARDS}} & 2 & 0.369 & 0.65 & \textbf{0.746} & 0.295 & 0.484 \\
 & 3 & 0.384 & NA & \multicolumn{1}{c}{0.553} & 0.319 & \multicolumn{1}{c}{\textbf{0.628}} \\
 & 4 & 0.4 & NA & \textbf{0.856} & 0.301 & 0.76 \\
\midrule

\multicolumn{1}{c}{\textbf{WID-M}} & 2 & 0.346 & \multicolumn{1}{c}{\textbf{0.797}} & \multicolumn{1}{c}{0.423} & 0.202 & \multicolumn{1}{c}{0.393} \\
 & 3 & 0.189 & NA & \textbf{0.45} & 0.157 & \multicolumn{1}{c}{0.282} \\
 & 4 & 0.136 & NA & \textbf{0.544} & 0.152 & \multicolumn{1}{c}{0.197} \\
\midrule

\multicolumn{1}{c}{\textbf{Diabetes}} & 2 & 0.249 & \multicolumn{1}{c}{\textbf{0.448}} & \multicolumn{1}{c}{0.378} & 0.156 & 0.129 \\
 & 3 & 0.201 & 0.181 & \textbf{0.503} & 0.167 & 0.154 \\
 & 4 & 0.157 & 0.088 & \textbf{0.423} & 0.183 & 0.24 \\
\midrule

\multicolumn{1}{c}{\textbf{CIC\_LOS}} & \multicolumn{1}{c}{2} & 0.547 & \multicolumn{1}{c}{\textbf{0.582}} & \multicolumn{1}{c}{0.546} & \multicolumn{1}{c}{-} & 0.231 \\
 & \multicolumn{1}{c}{3} & 0.306 & 0.372 & \textbf{0.447} & \multicolumn{1}{c}{-} & 0.099 \\
 & \multicolumn{1}{c}{4} & 0.162 & NA & \textbf{0.349} & \multicolumn{1}{c}{-} & 0.129 \\
\bottomrule

\end{tabular}
\end{adjustbox}
\caption{Silhouette Scores. Best values in bold.}
\label{tab:Sil_scores}
\end{table*}

\begin{table*}[!htb]
\centering
\begin{adjustbox}{max width=\textwidth}
\begin{tabular}{ccccccccc}
\toprule
\multicolumn{1}{l}{} & \multicolumn{1}{l}{} &  &  & \multicolumn{1}{l}{\textbf{}} & \textbf{F1} & \multicolumn{1}{l}{} & \multicolumn{1}{l}{} & \multicolumn{1}{l}{} \\
\midrule

\textbf{Dataset} & \textbf{k} & \multicolumn{1}{c}{\textbf{Baseline}} & \multicolumn{1}{c}{\textbf{SAE}} & \textbf{KMeans-Z} & \textbf{DCN-Z} & \textbf{IDEC-Z} & \textbf{DMNN-Z} & \textbf{\deepmethod} \\
\midrule

\textbf{CIC} & 1 & \multicolumn{1}{r}{\textbf{0.436}} & \multicolumn{1}{r}{0.487} & 0.617 & 0.629 & 0.628 & 0.299 & \textbf{0.678} \\
 & 2 &  &  & \textbf{0.639} & \multicolumn{1}{c}{NA} & 0.569 & 0.379 & \textbf{0.712} \\
 & 3 &  &  & \textbf{0.597} & \multicolumn{1}{c}{NA} & 0.606 & 0.413 & \textbf{0.693} \\
 & 4 &  &  & 0.432 & \multicolumn{1}{c}{NA} & 0.582 & 0.404 & \textbf{0.708} \\
\midrule

\textbf{Sepsis} & 1 & \multicolumn{1}{r}{0.500} & \multicolumn{1}{r}{0.498} & \textbf{0.715} & 0.646 & 0.684 & 0.157 & \textbf{0.806} \\
 & 2 &  &  & \textbf{0.561} & 0.525 & 0.686 & 0.427 & \textbf{0.822} \\
 & 3 &  &  & 0.557 & 0.59 & 0.653 & 0.363 & \textbf{0.814} \\
\textbf{} & 4 &  &  & \textbf{0.504} & \multicolumn{1}{c}{NA} & 0.612 & 0.537 & \textbf{0.815} \\
\midrule

\textbf{AKI} & 1 & \multicolumn{1}{r}{0.506} & \multicolumn{1}{r}{0.417} & \textbf{0.57} & 0.528 & 0.578 & 0.382 & \textbf{0.638} \\
 & 2 &  &  & 0.333 & 0.529 & 0.527 & \textbf{0.645} & 0.642 \\
\textbf{} & 3 &  &  & \textbf{0.459} & 0.569 & 0.541 & \textbf{0.647} & 0.625 \\
 & 4 &  &  & 0.361 & 0.571 & \textbf{0.574} & \textbf{0.640} & 0.605 \\
\midrule

\textbf{ARDS} & 1 & \multicolumn{1}{r}{0.481} & \multicolumn{1}{r}{0.481} & \textbf{0.561} & 0.529 & 0.525 & 0.043 & 0.554 \\
\textbf{} & 2 &  & \textbf{} & 0.499 & 0.498 & 0.483 & 0.132 & \textbf{0.567} \\
 & 3 &  &  & \textbf{0.555} & \multicolumn{1}{c}{NA} & 0.492 & 0.143 & \textbf{0.564} \\
 & 4 &  &  & \textbf{0.548} & \multicolumn{1}{c}{NA} & 0.481 & 0.163 & \textbf{0.551} \\
\midrule

\textbf{WID-M} & 1 & \multicolumn{1}{r}{0.520} & \multicolumn{1}{r}{\textbf{0.514}} & 0.644 & 0.581 & 0.65 & 0.026 & \textbf{0.696} \\
 & 2 &  &  & \textbf{0.604} & 0.497 & 0.621 & 0.331 & \textbf{0.694} \\
 & 3 &  &  & \textbf{0.582} & \multicolumn{1}{c}{NA} & 0.616 & 0.331 & \textbf{0.686} \\
 & 4 &  &  & 0.578 & \multicolumn{1}{c}{NA} & 0.606 & 0.360 & \textbf{0.686} \\
\midrule

\textbf{Diabetes} & 1 & \multicolumn{1}{r}{0.471} & \multicolumn{1}{r}{0.438} & \textbf{0.491} & 0.436 & 0.486 & 0.226 & 0.465 \\
 & 2 &  &  & \textbf{0.438} & 0.421 & 0.468 & 0.235 & \textbf{0.493} \\
 & 3 &  &  & \textbf{0.468} & 0.401 & 0.439 & 0.280 & 0.447 \\
 & 4 &  &  & 0.453 & 0.393 & 0.451 & 0.279 & \textbf{0.46} \\
\midrule

\textbf{CIC-LOS} & 1 & \multicolumn{1}{r}{0.399} & \multicolumn{1}{r}{0.427} & 0.433 & 0.403 & 0.431 & \multicolumn{1}{c}{-} & \textbf{0.458} \\
 & 2 &  &  & 0.354 & 0.371 & 0.409 & \multicolumn{1}{c}{-} & \textbf{0.456} \\
 & 3 &  &  & 0.324 & 0.394 & 0.403 & \multicolumn{1}{c}{-} & \textbf{0.46} \\
 & 4 &  &  & 0.292 & \multicolumn{1}{c}{NA} & 0.398 & \multicolumn{1}{c}{-} & \textbf{0.442} \\
\bottomrule
\end{tabular}
\end{adjustbox}
\caption{F1 Scores. Best values in bold.}
\label{tab:F1_scores}
\end{table*}

\newpage
\section{Extended Ablation Analysis}
\label{app:extended_ablation_analysis}

\begin{figure*}[!htb]
     \centering
     \begin{subfigure}[b]{0.33\textwidth}
         \centering
         \includegraphics[width=\textwidth]{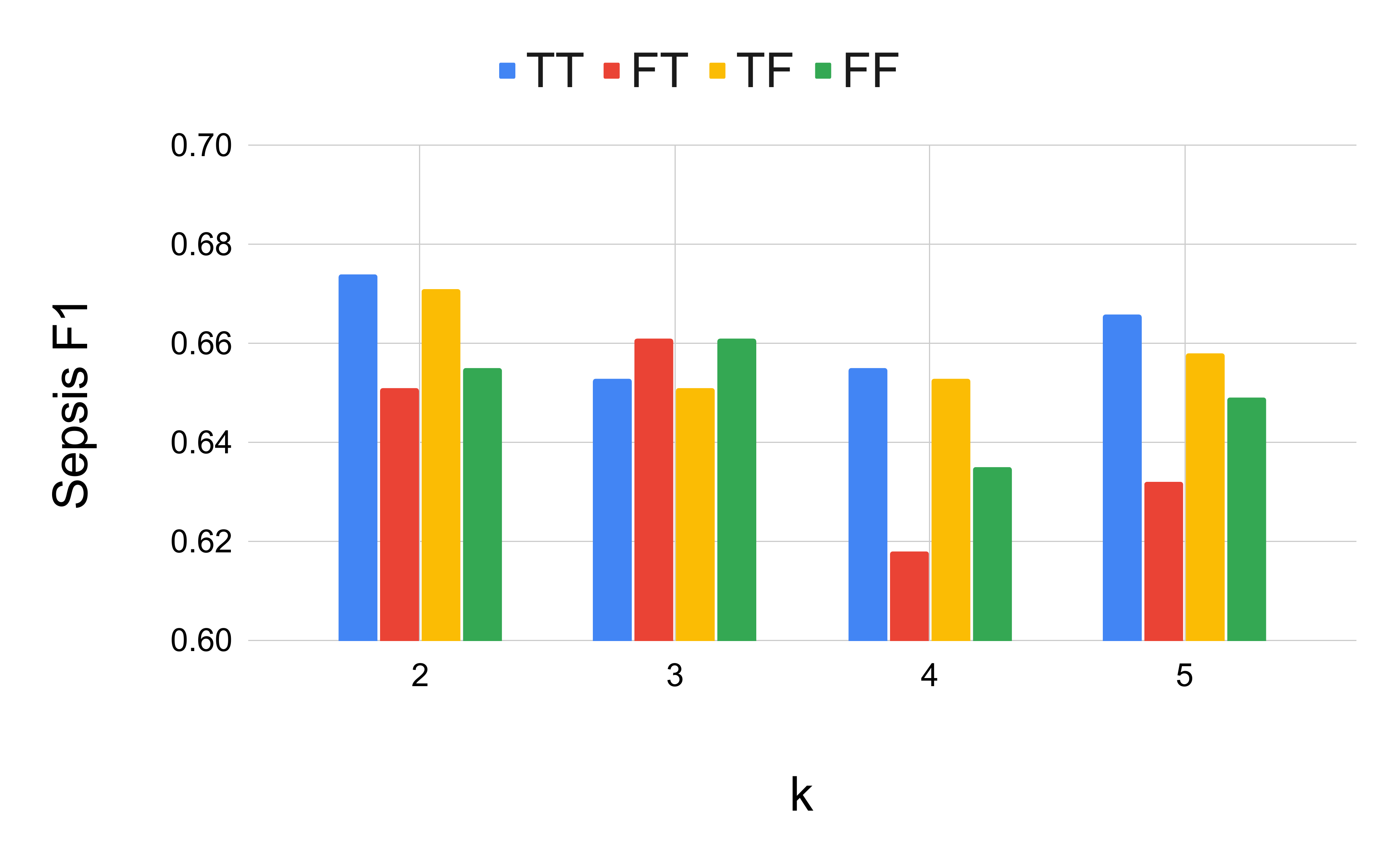}
        \caption{F1/AUC vs. $\beta$}
        \label{fig:app_f1_auc_beta}
    \end{subfigure}
     \hfill
     \begin{subfigure}[b]{0.33\textwidth}
         \centering
         \includegraphics[width=\textwidth]{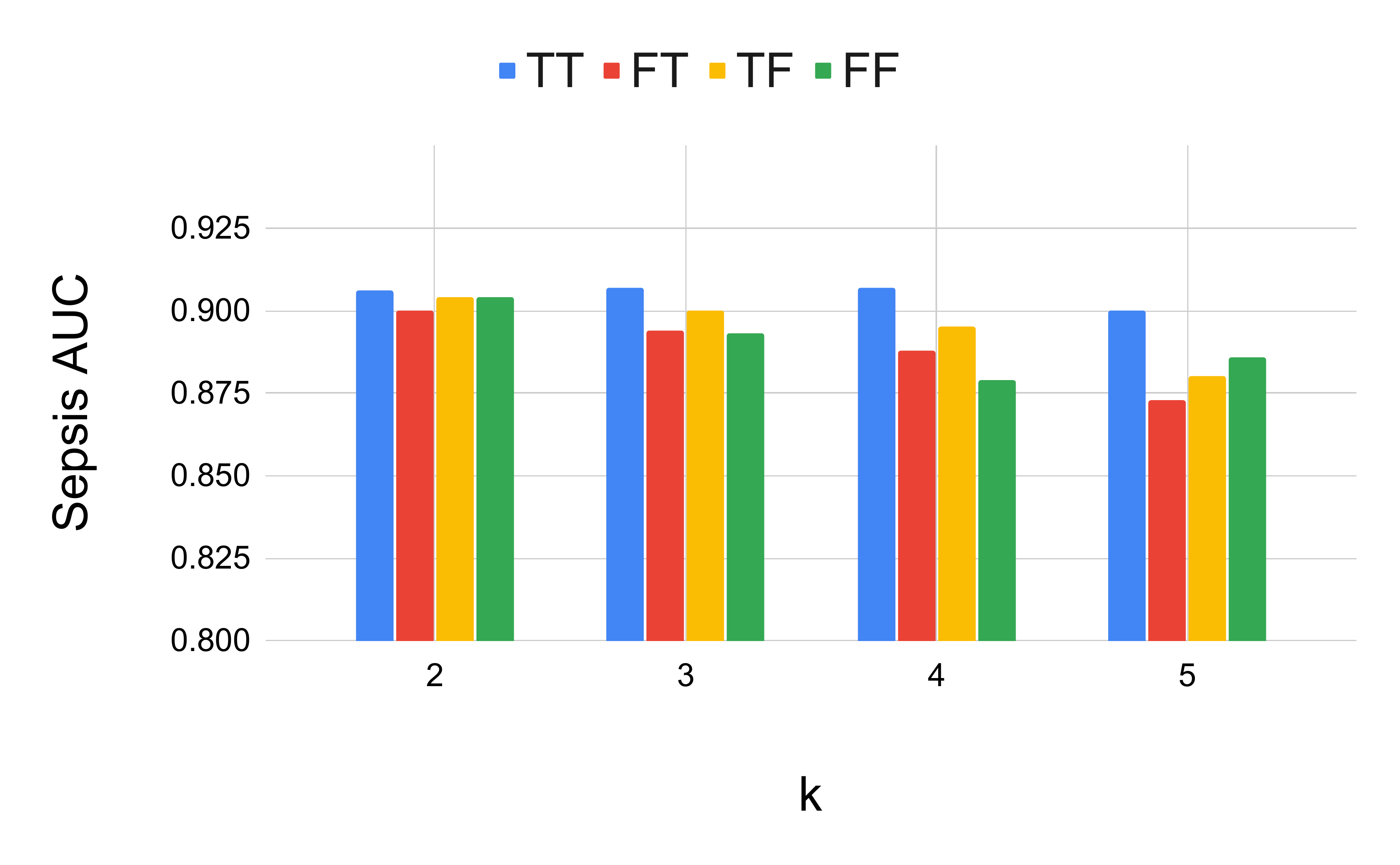}
        \caption{F1/AUC vs. $\gamma$}
        \label{fig:app_f1_auc_gamma}
     \end{subfigure}
     \hfill
     \begin{subfigure}[b]{0.33\textwidth}
         \centering
         \includegraphics[width=\textwidth]{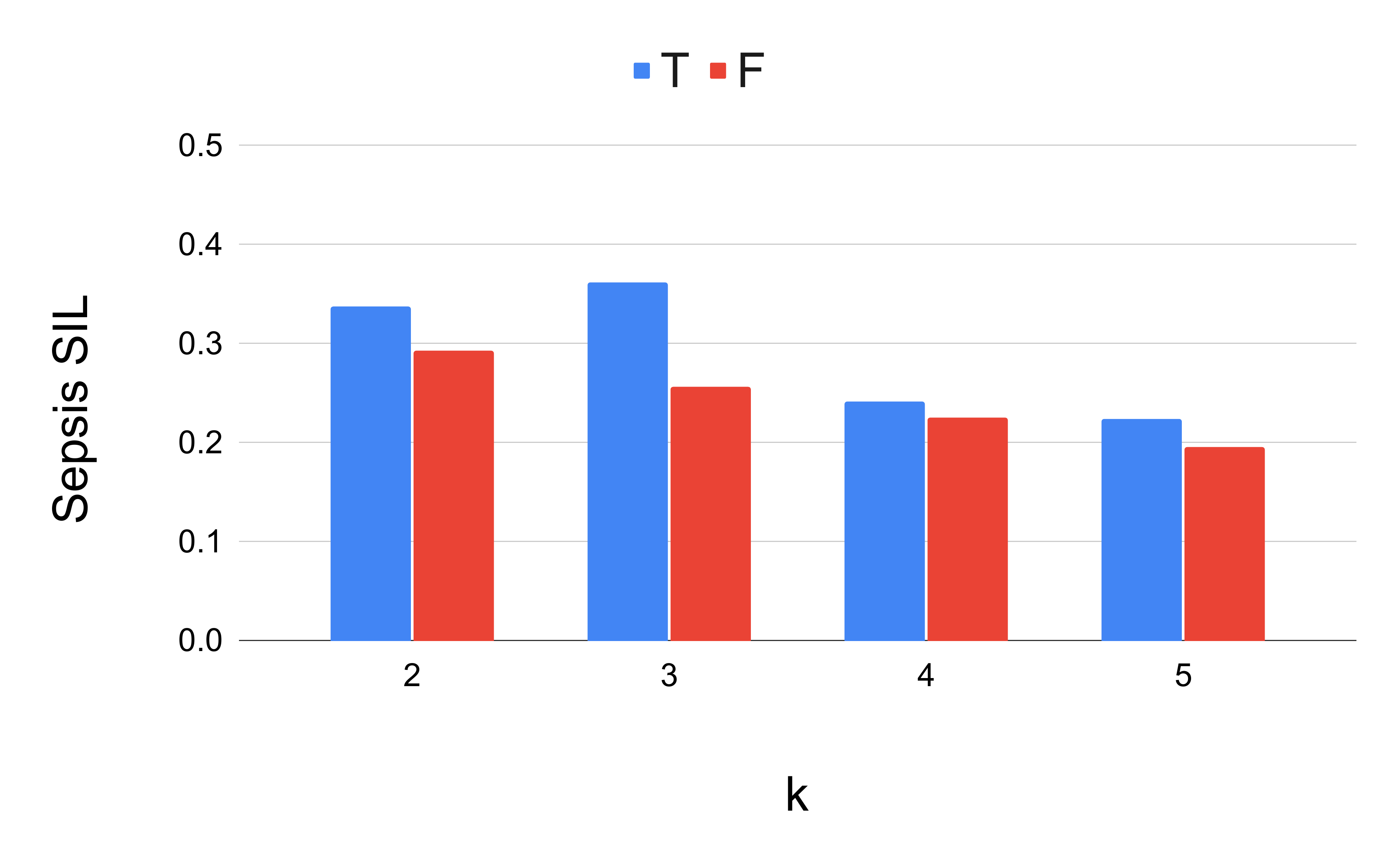}
         \caption{SIL}
         \label{fig:app_sil_cic}
     \end{subfigure}
    \caption{Ablation study on Sepsis dataset. TT represents that stochastic sampling is turned on in both training and testing phase. Similarly for other combinations. Note that stochastic sampling turned on in the training phase has a strong effect on performance.}
    \label{app:fig:sepsis_ablation}
\end{figure*}

\section{Data Preprocessing and Feature Extraction}
\label{app:data_preprocessing}

We evaluate the results on $6$ real world clinical datasets out of which $3$ are derived from the \textsc{MIMIC} III dataset derived from Beth Israel Deaconess Hospital. CIC dataset is derived from 2012 Physionet challenge \cite{silva2012cic} of  predicting in-hospital mortality of intensive care units. The Sepsis dataset is derived from the 2019 Physionet challenge of sepsis onset prediction. We manually extract the Acute Kidney Failure (AKI) and Acute Respiratory Distress Syndrome (ARDS) datasets from the larger MIMIC III dataset. WID Mortality dataset is extracted from the 2020 Women In Data Science challenge to predict patient mortality. The Diabetes Readmission prediction dataset consists of $100000$ records of patients. This is a multiclass version of the CIC dataset where the Length of Stay (LOS) feature is discretized into $3$ classes using quartiles. The task is to predict the LOS of ICU patients. 

\begin{itemize}
    \item \textbf{CIC}: The CIC dataset is derived from the 2012 Physionet challenge \cite{silva2012cic} of  predicting in-hospital mortality of intensive care units (ICU) patients at the end of their hospital stay. The data has time-series records, comprising various physiological parameters of $12000$ patient ICU stays. We follow the data processing scheme of \cite{johnson2012patient} (the top ranked team in the competition) to obtain static 117-dimensional features for each patient. 
    
    \item \textbf{Sepsis}: The Sepsis dataset is derived from the 2019 Physionet challenge of sepsis onset prediction. The dataset has time-series records, comprising various physiological parameters of $\sim40000$ patients. We follow the data processing scheme of the competition winners \cite{morrill2019signature} to obtain static 89 dimensional features for each patient.
    
    \item \textbf{AKI and ARDS}: We manually extract the Acute Kidney Failure (AKI) and Acute Respiratory Distress Syndrome (ARDS) datasets from the larger MIMIC III dataset. We follow the KDIGO criteria
    to determine kidney failure onset time and the Berlin criteria 
    for ARDS. The challenge is to predict kidney failure onset ahead of time. Similar to the above datasets, we derive static vectors from the time series data.
    
    \item \textbf{WID Mortality}: This dataset is extracted from the 2020 Women In Data Science challenge where the objective is to create a model that uses data from the first 24 hours of intensive care to predict patient survival. We perform standard data cleaning procedures before using the dataset.

    \item \textbf{Diabetes}: The Diabetes Readmission prediction dataset consists of $100000$ records of patients from 130 hospitals in the US from the year 1998 to 2008.
    
    \item \textbf{CIC-LOS}: This is a multiclass version of the CIC dataset where the Length of Stay (LOS) feature is discretized into $3$ classes using quartiles. The task is to predict the LOS of ICU patients.
\end{itemize}

\section{\deepmethod\ Optimization Details}
\label{app:optimization}

\subsection{Updating \deepmethod\ parameters}

Optimizing $(\Ucal, \Vcal)$ is similar to training an SAE - but with the additional loss terms $L_c, L_s$ and $L_{bal}$. We code our algorithm in PyTorch which allows us to easily backpropagate all the losses simultaneously. To implement SGD for updating the network parameters, we look at the problem w.r.t. the incoming data $x_i$:

\begin{align*}
\min_{\Ucal, \Wcal} L^{i} = \ell(g(f(x_i)), x_i) + \beta \cdot \operatorname{KL}(P,Q) \\
+ \gamma \cdot \sum_{j=1}^{k} q_{i,j}\operatorname{CE}(y_i, h_j(x_i; \mathcal{V}_j)) + \delta \cdot L_{bal}
\end{align*}

The gradient of the above function over the network parameters is easily computable. Let $\Jcal=(\Ucal, \Vcal, \Wcal)$ be the collection of network parameters, then for a fixed target distribution $P$, the gradients of $L_c$ w.r.t. embedded point $z_i$ and cluster center $\mu_j$ can be computed as:

\begin{align*}
&\frac{\partial L_{c}}{\partial z_{i}}=2 \sum_{j=1}^{k}\left(1+\left\|z_{i}-\mu_{j}\right\|^{2}\right)^{-1}\left(p_{i j}-q_{i j}\right)\left(z_{i}-\mu_{j}\right) \\
&\frac{\partial L_{c}}{\partial \mu_{j}}=2 \sum_{i}\left(1+\left\|z_{i}-\mu_{j}\right\|^{2}\right)^{-1}\left(q_{i j}-p_{i j}\right)\left(z_{i}-\mu_{j}\right)
\end{align*}

The above derivations are from \cite{xie2016unsupervised}. We leverage the power of automatic differentiation to calculate the gradients of $L_{bal}$ and $L_s$ during execution. Then given a mini batch with $m$ samples and learning rate $\lambda, \mu_j$ is updated by:

\begin{align}
    \label{eqn:mu_update}
    \mu_j = \mu_j - \frac{\tau}{m} \sum_{i=1}^{m}\frac{\partial L_c}{\partial \mu_j}
\end{align}

The decoder's weights are updated by:
\begin{align}
    \label{eqn:V_update}
    \Vcal = \Vcal - \frac{\tau}{m} \sum_{i=1}^{m}\frac{\partial L_r}{\partial \Vcal}
\end{align}

and the encoder's weights are updated by:
\begin{align}
    \label{eqn:U_update}
    \Ucal = \Ucal - \frac{\tau}{m} \sum_{i=1}^{m}\left(\frac{\partial L_r}{\partial \Ucal} + \beta\frac{\partial L_c}{\partial \Ucal} + \gamma\frac{\partial L_s}{\partial \Ucal} + \delta \frac{\partial L_{bal}}{\partial \Ucal}\right)
\end{align}

\[
\Jcal \xleftarrow{} \Jcal - \tau \nabla_{\Jcal} L^i
\]

where $\tau$ is the diminishing learning rate.\\

\section{Case Study: Additional Analysis}
\label{app:cic_clusters}

\begin{table}[!h]
\centering
\begin{adjustbox}{max width=0.5\textwidth}
\begin{tabular}{lll}
\toprule
Cluster 1 & Cluster 2 & Cluster 3 \\ \midrule
SAPS-I & Length\_of\_stay & SAPS-I \\
SOFA & CCU & SOFA \\
GCS\_first & CSRU & Length\_of\_stay \\
GCS\_lowest & DiasABP\_first & Weight \\
CSRU & GCS\_first & CSRU \\
Creatinine\_last & Creatinine\_last & MechVentLast8Hour \\
Creatinine\_last & Glucose\_first & DiasABP\_first \\
GCS\_last & HR\_first & Glucose\_first \\
BUN\_last & MAP\_first & HR\_first \\
Creatinine\_first & NIDiasABP\_first & MAP\_first \\
Lactate\_first & NIMAP\_first & NIDiasABP\_first \\
\hline
\end{tabular}
\end{adjustbox}
\caption{Cluster wise features sorted by $\operatorname{HTFD}$ values.}
\label{tab:fd_metrics}
\end{table}

We study the clusters found by \deepmethod\ by analyzing the $\operatorname{HTFD}$ metric in table \ref{tab:fd_metrics}. SAPS-I and SOFA have high $\operatorname{HTFD}$ in clusters $1$ and $3$. This indicates that the two clusters have different distributions of SAPS-I and SOFA. It (SOFA) is also an important feature for cluster $3$'s risk model (Table \ref{tab:important_features}) but in distribution, it is similar to SOFA values of clusters $2$ but not $1$. GCS\_Last (Glasgow Coma Score) is an important feature for all the three risk models but not significantly different across the three clusters (cluster averages $12.744, 14.489, 10.226$). Normal GCS is $15$.

\end{document}